%% file: VSRKHS_arxiv.tex
\documentclass{article}

\usepackage{arxiv}

\usepackage[utf8]{inputenc} 
\usepackage[T1]{fontenc}    
\usepackage{hyperref}       
\usepackage{url}            
\usepackage{booktabs}       
\usepackage{amsfonts}       
\usepackage{nicefrac}       
\usepackage{microtype}      
\usepackage{lipsum}		
\usepackage{graphicx}
\usepackage{natbib}
\usepackage{doi}
\usepackage{natbib}
\usepackage{fancyhdr,setspace,epsf,subfigure,rotate}
\usepackage{amssymb,amsmath,dsfont, mathtools,amsthm}
\usepackage{lineno}
\usepackage{ctable, multirow}
\usepackage{mathrsfs}
\usepackage{geometry}
\usepackage{algorithm,xcolor}
\usepackage{caption}
\usepackage{lineno}
\usepackage{ctable, multirow}
\usepackage{booktabs}
\usepackage{mathrsfs}
\usepackage{rotating}

\newcommand{\Bayes}{{\mbox{\scriptsize{Bayes}}}}
\newcommand{\old}{{\mbox{\scriptsize{old}}}}

\newcommand{\sumi}{\sum_{i=1}^{n}}
\newcommand{\sumj}{\sum_{j=1}^{n}}

\newcommand{\T}{\top}

\newcommand{\bone}{\mathds{1}}

\def\argmin{\mathop{\rm argmin}}

\newcommand{\bH}{\mathbf{H}}
\newcommand{\bK}{\mathbf{K}}

\newcommand{\bg}{\mathbf{g}}
\newcommand{\bx}{\mathbf{x}}

\newcommand{\bW}{\mathbf{W}}
\newcommand{\bX}{\mathbf{X}}

\newcommand{\bk}{\mathbf{k}}

\newcommand{\bu}{\mathbf{u}}

\newcommand{\bbeta}{\boldsymbol{\beta}}

\newcommand{\balpha}{\boldsymbol{\alpha}}

\newcommand{\calF}{\mathcal{F}}

\newcommand{\calH}{\mathcal{H}}

\newcommand{\calE}{\mathcal{E}}
\newcommand{\calR}{\mathcal{R}}

\newcommand{\calX}{\mathcal{X}}
\newcommand{\Rho}{\mathrm{P}}

\def\argmin{\mathop{\rm argmin}}

\theoremstyle{definition}

\theoremstyle{remark}

\theoremstyle{definition}

\newtheorem{theorem}{Theorem}
\newtheorem{lemma}[theorem]{Lemma}
\newtheorem{corollary}[theorem]{Corollary}

\newtheorem{definition}{Definition}

\title{A Gradient-Based Variable Selection for Binary Classification in Reproducing Kernel Hilbert Space}

\author{Jongkyeong Kang \\
	Department of Statistics\\
	Seoul National University\\
	08826, Seoul, Republic of Korea\\
	\texttt{coolnessjk@snu.ac.kr} \\
	\And
	Seung Jun Shin \\
	Department of Statistics\\
	Korea University\\
	02841, Seoul, Republic of Korea\\
	\texttt{sjshin@korea.ac.kr} \\
}

\date{}

\renewcommand{\headeright}{}
\renewcommand{\undertitle}{}
\renewcommand{\shorttitle}{Variable Selection for Binary Classification in RKHS}

\begin{document}
	\maketitle
	
	\begin{abstract}
Variable selection is essential in high-dimensional data analysis. Although various variable selection methods have been developed, most rely on the linear model assumption. This article proposes a nonparametric variable selection method for the large-margin classifier defined by reproducing the kernel Hilbert space (RKHS). we propose a gradient-based representation of the large-margin classifier and then regularize the gradient functions by the group-lasso penalty to obtain sparse gradients that naturally lead to the variable selection. The groupwise-majorization-decent \citep[GMD,][]{yang2015fast} algorithm is proposed to efficiently solve the proposed problem with a large number of parameters. We employ the strong sequential rule \citep{tibshirani2012strong} to facilitate the tuning procedure. The selection consistency of the proposed method is established by obtaining the risk bound of the estimated classifier and its gradient. Finally, we demonstrate the promising performance of the proposed method through simulations and real data illustration.
	\end{abstract}



\section{Introduction}
Variable selection is an essential task to optimize prediction performance and interpretability of the model and has a long history, especially for linear regression. Sequential methods such as forward addition and backward elimination still are popular choices in practice. The best subset selection is desirable in theory but may not be applicable in practice unless the number of predictors is restrictively small. 

After the seminal LASSO \citep{tibshirani1997lasso} proposed, the penalization that yields a sparse estimator has gained great popularity in the variable selection for the linear model. In addition to the canonical $L_1$-penalty, various penalties that lead to sparse estimators have developed. Popular examples include, but are not limited to, adaptive LASSO \citep{zou2006adaptive}, elastic net \citep{zou2005regularization}, group LASSO \citep{yuan2006model}, and SCAD \citep{Fan::2001scad} penalties. The penalized variable selection can readily be extended to the linear classifier since both regression and classification can be viewed as an empirical risk minimization (ERM) problem under which penalization is natural. We refer to \citet{hastie2015statistical} and \citet{buhlmann2011statistics} and references therein for a comprehensive overview of penalized variable selection for both regression and classification.

In this article, we study the variable selection for nonlinear binary classification, which has become rapidly popular in contemporary applications in statistics and data science. Although the penalized variable selection is known to be promising in both theory and practice, its success is heavily based on the linear model assumption, where a sparse coefficient vector directly leads to variable selection. However, the variable selection for the nonlinear classifier may not be straightforward and it has not been explored extensively as for the linear model. \citet{lin2006component} proposed the COSSO, a natural extension of LASSO to the smoothing-spline ANOVA model \citep{gu2013smoothing}, a well-established nonlinear model. \citet{dasgupta2019feature} proposed a sequential variable selection algorithm for the nonlinear models defined in reproducing kernel Hilbert space \citep[RKHS,][]{Wahba:1999RKHS}. Finally, motivated by attenuation in the measurement error model, \citet{stefanski2014variable} developed a measurement error model-based approach to variable selection for nonparametric classification.

In the regression context, \cite{yang2016model} proposed a novel variable selection method based on learning gradients that contain meaningful information for variable selection. As the regression coefficient is a gradient vector of the linear regression model, the sparsity of the gradient vector naturally leads to the selection of variables even for nonlinear models.  Motivated by \cite{mukherjee2006estimation} and \citet{ye2012learning}, \cite{yang2016model} reformulate the regression problem to gradient-based ERM problem in RKHS and then employ the group-LASSO type penalty to pursue sparse gradient vector to achieve the variable selection. \citet{he2018gradient} extended this gradient-based idea to variable selection for nonparametric quantile regression.

Although the gradient-based approach for variable selection is based on the ERM framework, its extension to classification is not a straightforward extension. The gradient-based formulation takes advantage of the first-order Taylor expansion. In the regression with the continuous response, the constant term can be directly approximated by the response, while it must be estimated in binary classification. That is, we need to estimate both the classifier and its gradient simultaneously, which causes technical bottlenecks in theory and computation. \citet{he2020variable} employed the derivative reproducing kernel \citep{zhou2008derivative} to extend the gradient-based variable selection to the classification. The proposed method is computationally much more efficient than the original formulation proposed by \citet{yang2016model}. Yet, \citet{zhou2008derivative} only showed a sure screening property and failed to prove the selection consistency, the minimal condition that any reasonable variable selection methods must satisfy.

In this article, we extend the gradient-based variable selection \citep{yang2016model} to propose a consistent variable selection for nonlinear classification in RKHS. To avoid the heavy computation involving a large number of parameters, we apply the groupwise-majorization-decent \citep[GMD,][]{yang2015fast} algorithm. Furthermore, we modified the strong sequential rule \citep{tibshirani2012strong} to facilitate the selection of tuning parameters.

The rest of the paper is organized as follows. In Section \ref{sec2}, we present our proposed variable selection method for binary classification in RKHS. Section \ref{sec3} is devoted to describing how both the groupwise majorization-decent algorithm and the strong sequential rule are applied to solve the proposed method with a large number of parameters. In Section \ref{sec4}, we study the asymptotic properties of the proposed method to establish the selection consistency. Numerical experiments and real data applications are presented in Sections \ref{sec5} and \ref{sec6}, respectively. Concluding remarks are given in Section \ref{sec7}. Finally, the technical proofs are relegated to Appendix. 

\section{Gradient learning for classification}\label{sec2}

\subsection{Large-margin Classifier}

Let $\bX \in \mathbb{R}^p$ and $Y = \{-1, 1\}$ be a $p$-dimensional predictor and a binary response random variable drawn from $\Rho(\bx,y)$ the joint distribution of $(\bX,Y)$. The objective of the classification is to learn a classifier $f$ that minimizes the population misclassification error rate, $P\{Y f(\bX) \le 0 \}$. We call such a classification function Bayes optimal classifier, denoted by $f^{\Bayes}$ which yields an optimal classification rule for a given $\bx$. To be more precise, the Bayes classifier $f^*$ is defined by a solution of the following optimization problem:
\begin{align} \label{eq:loss0}
	f^\Bayes = \argmin_{f\in \calH} \int \bone_{\{x \le 0\}}\left(y f(\bx) \right)d \Rho(\bx, y),
\end{align}
where $\calH$ is the Hilbert space where a function $f$ resides, and $\bone_A(m)$, often referred to as the zero-one loss, denotes an indicator function that takes 1 if $m \in A$ and 0 otherwise. Despite its optimality, it is not easy to solve \eqref{eq:loss0} directly due to the irregularity of the zero-one loss function.

To circumvent this, one considers a convex relaxation of zero-one loss, denoted with $L: \mathbb{R} \rightarrow \mathbb{R}_{+} \cup\{0\}  $ as follows:
\begin{align} \label{eq:loss}
	f^* = \argmin_{f \in {\calH}} \int L(yf(\bx))d\Rho(\bx,y)
\end{align}

Under mild conditions \citep{vapnik1998support,bartlett2005local}, $f^*$ yields an identical classification rule to $f^{\Bayes}$, i.e., $\text{sign}(f^*(\bx)) = \text{sign}(f^\Bayes(\bx))$. If this holds, we say that $L$ is Fisher consistent (or classification calibrated). A popular choice of $L$ includes hinge loss function $L(m) = [1 - m]_+$ for the support vector machine where $[a]_+ = \max \{0, a\}$, logistic loss function $L(m) = \log\{1 + \exp(-m) \}$ for the logistic regression, and the exponential loss function $L(m) = \exp(-m)$ for the additive boosting, to name a few. 
Suppose that we are given a random sample of size $n$, $\{(\bx_i,y)\}_{i=1}^{n}$, drawn from $\Rho(\bx,y)$. A sample version of \eqref{eq:loss}, along with an additional regularization of $f$, is.
\begin{equation} \label{eq:large.margin}
	\hat f = \argmin _{f \in \calH}\frac{1}{n} \sum_{i=1}^{n} L\left(y_{i} f\left(\bx_{i}\right)\right) + \lambda J (f),
\end{equation}
where $J$ denotes a penalty functional that measures the roughness of $f$, and $\lambda> 0$ is a tuning parameter that controls the balance between the data fitting and the complexity of the classifier. We call \eqref{eq:large.margin} the large-margin classifier, the most popular binary classification tool in the statistical learning community.

\subsection{Gradient-Based Formulation}

Assume that a gradient vector of $f^*$, denoted by $\nabla f^{*}=({\partial f^{*}}/{\partial x_{1}}, \ldots, {\partial f^{*}}/{\partial x_{p}})^\T$ exist.
Consider the first order Taylor expansion of $f^*$ at around $\bx = \bu$:
\begin{align} \label{appx}
	f^*(\bx) \approx f^*(\bu)+\nabla f^*(\bu)^\T(\bx-\bu).
\end{align}
By \eqref{appx}, one can approximate the objective function of \eqref{eq:loss} as follows:
\begin{equation}\label{eq:def1}
	\hat \calE(f, \bg)=\frac{1}{n^{2}} \sum_{i, j=1}^{n} \omega_{s}(\bx_{i}-\bx_{j}) L\left(y_{i}\left(f\left(\bx_{j}\right)+\bg\left(\bx_{j}\right)^\T \left(\bx_{i}-\bx_{j}\right)\right)\right)
\end{equation}
for some functions $f : \calX \to \mathbb{R}$ and $\bg : \calX \to \mathbb{R}^p$ corresponding to $f^*$ and $\nabla f^*$, respectively, with $\calX$ being the support of $\bX$. Here 
$\omega_{s}$ denotes a smoothing kernel function with a bandwidth parameter $s$. In this paper, we focus on a Gaussian kernel $\omega_{s}(\bx-\bu)=s^{-(p+2)}\exp\left\{-\frac{1}{2s^2}\|\bx-\bu\|^2\right\}$ due to its popularity and computational simplicity.

Plugging \eqref{eq:def1} into \eqref{eq:large.margin}, we have the gradient-based form of the large-margin classifier as follows.
\begin{align} \label{eq1}
	\min_{(f, \bg) \in \calH_K^{p+1}}\hat \calE(f, \bg) +  \lambda  \{J_0(f)+ J_1(\bg)\},
\end{align}
where $J_0$ and $J_1$ denote the penalty functionals that measure the complexity of $f$ and $\bg$, respectively.

Toward the variable selection, we need to design $J_1$ carefully to have a sparse graidnet vector. Given $\{\bx_i, y_i\}_{i=1}^n$, we assume that each component of $\bg = (g_1, \cdots, g_p)^\T$ as well as $f$ live on the reproducing kernel Hilbert space (RKHS) $\calH_{K}$  generated by a positive definite kernel $K(\cdot,\cdot)$. By \textit{Representer Theorem} \citep{kimeldorf1971some}, we have the following finite-dimensional representation for both $f$ and $\bg$: 
\begin{align*}
	f(\bx)= \sum_{i=1}^{n}{\alpha}_{i0}K(\bx,\bx_i) \quad \mbox{and} \quad
	g_\ell(\bx)= \sum_{i=1}^{n} \alpha_{i\ell} K(\bx, \bx_i), ~ \ell = 1, \cdots, p.
\end{align*}
This yields
\begin{align} \label{eq2}
	f(\bx_j) + \bg(\bx_j)^\T (\bx_i - \bx_j) 
	&= \balpha_0^\T\bk_j+\sum_{\ell=1}^{p}\balpha_{\ell}^{\T} \bk_{j}\delta_{ij\ell},
\end{align}
where $\delta_{ij\ell}= x_{i\ell}-x_{j\ell}$ for $\ell=1,\cdots ,p$, and $\bk_j$ is the $j$-th column of the kernel matrix $\bK$. It is clear that $g_\ell(\bx)\equiv0$ if and only if $\balpha_{\ell}\equiv0$ or equivalently $\|\balpha_\ell\|_2=0$. This leads us to consider 
a group-Lasso type penalty for $J_1(\bg) =\sum_{\ell=1}^{p}\theta_\ell
J(g_\ell)$ with
\begin{align} \label{glpen}
	J(g_\ell)=\inf \left\{\|\balpha_{\ell}\|_2 : g_\ell(\cdot)=\sum_{i=1}^{n}\alpha_{i\ell}K(\cdot,\bx_i)	\right\}.
\end{align}
Here, $\theta_\ell > 0$ adaptively controls the contribution of the $\ell$th gradient as done by the adaptive LASSO \citep{zou2006adaptive}. 
The infimum in \eqref{glpen} is required because the representation of $g_\ell \in \calH_{K}$ may not be unique. %
We remark that this coefficient-based group lasso type of penalty $\|\balpha_{\ell}\|_2$ achieves 
the variable selection by forcing all $\alpha_{i\ell}$'s to exactly zero simultaneously, and the penalty \eqref{glpen} does not depend on the choice of the kernel $K$. In terms of variable selection, the choice of $J_0$ is not essential, and  
we employ a conventional ridge type penalty for $J_0(f)=\frac{\theta_0}{2} \| \balpha_{0} \|_2^2 $. 

Putting \eqref{eq2} into \eqref{eq1} with the aforementioned choice of $J_0$ and $J_1$, we finally propose to solve the following problem of variable selection in binary classification defined in RKHS. 
\begin{align} \label{eq:gradient}
	(\hat{\balpha}_{0},\cdots, \hat{\balpha}_{\ell})=\argmin_{({\balpha}_{0},\cdots, {\balpha}_{\ell}) }&
	\frac{1}{n^2}	\sumi \sumj  \omega_{s}(\bx_{i}-\bx_{j})L\{y_i \sum_{\ell=0}^{p}\balpha_{\ell}^{\T} \bk_{j}\delta_{ij\ell}\}  
	+	\lambda\left(\frac{\theta_0}{2}\|\balpha_{0}\|_2^2 +   \sum_{\ell=1}^{p}\theta_\ell\|\balpha_{\ell}\|_2\right).
\end{align}
Although any Fisher consistent $L$ can used for \eqref{eq:gradient}, in this article, we focus on two popular loss functions: the logistic loss $L(m) = \log ( 1 + \exp(-m))$ and the squared hinge loss $L(m) = [1-m]_+^2$ for technical simplicity. See Section \ref{sec7} for related discussions of the extension to other popular loss functions such as hinge loss.

\section{Computing algorithm} \label{sec3}
\subsection{GMD algorithm}
Let $\balpha=(\balpha_{0}^\T,\cdots,\balpha_{p}^\T)^\T$ and  $\tilde{\bK}_{ij}=(\delta_{ij0}\bk_{j}^\T,\cdots, \delta_{ijp}\bk_{j}^\T)^\T$. Then we define $L(\balpha)$ to denote the loss term in \eqref{eq:gradient} as follows. 
\begin{align} \label{unpen}
	L(\balpha)= \frac{1}{n^2} \sumi\sumj \omega_{s}(\bx_{i}-\bx_{j})L(y_i\tilde{\bK}_{ij}^\T{\balpha}).
\end{align}

We implement the groupwise-majorization-descent (GMD) algorithm \citep{yang2015fast} to solve \eqref{eq:gradient}. We remark that the GMD algorithm works only when the loss function satisfies the quadratic majorization (QM) condition. 
\begin{definition}
	A differentiable loss function $L$ is said to satisfy the quadratic majorization (QM) condition, if and only if 
	there exists a $n(p+1)\times n(p+1)$ matrix $\bH$ such that 
	\begin{align} \label{eq:qm}
		L(\balpha) \le L(\balpha^{\prime})+(\balpha-\balpha^\prime)\nabla L(\balpha^\prime)+\frac{1}{2}(\balpha-\balpha^\prime)^\T\bH(\balpha-\balpha^\prime), \qquad \forall \balpha, \balpha^\prime.
	\end{align}
\end{definition}
Note that the GMD algorithm is essentially a majorize-minimization (MM) algorithm \citep{hunter2004tutorial}, and the QM condition ensures that the empirical loss $L(\balpha)$ is majorized by the quadratic function given in \eqref{eq:qm}.

Let $\bW$ be an $n^2$-dimensional diagonal matrix whose $n(i-1)+j$th diagonal element is $\omega_{s}(\bx_{i}-\bx_{j})$, and $\tilde \bK$ be an $n^2 \times n(p+1)$ matrix whose $n(i-1)+j$th row is $\tilde \bK_{ij}$, for $i, j = 1, \cdots, n$. It turns out by Lemma 1 of \citet{yang2015fast} that both the logistic loss and squared hinge loss satisfy the QM condition with
\begin{align} \label{eq:H}
	\bH = \frac{C}{n^2} \tilde\bK^\T \bW \tilde \bK,
\end{align}
and $C = 1/4$ for logistic loss and $C = 1$ for squared hinge loss. See \citet{yang2015fast} for complete details about the GMD algorithm.

Let ${\balpha}^{\old}$ denote the current value of $\balpha$. In order to update $\balpha_{\ell}$, let us set $\balpha$ to be identical to $\balpha^\old$ except for the $\ell$th column, $\balpha_{\ell}$ which is regarded as unknown parameter to be updated. 
By \eqref{eq:qm}, we have 
\begin{align}\label{logeq1}
	L(\balpha) \le L({\balpha}^\old) &+ (\balpha_{\ell} - \balpha^\old_{\ell})^\T  \nabla L(\balpha^\old)_{\ell} + \frac{1}{2}(\balpha_{\ell}-\balpha^\old_{\ell})^\T\bH_{\ell}(\balpha_{\ell}-\balpha^\old_{\ell}), ~ \ell = 0, 1, \cdots, p,
\end{align}
where $\nabla L(\balpha)_{\ell}$ is the $\ell$-th block component of $\nabla L(\balpha)$ and $\bH_{\ell}$ is the $\ell$-th block diagonal matrix of $\bH$ given in \eqref{eq:H}. 
We can easily show that the $(a,b)$-th component of $\bH_{\ell}$ is
$\sum_{i=1}^n\sum_{j=1}^n  \omega_{s}(\bx_{i}-\bx_{j})\delta_{ij\ell}^2 K_{aj}K_{bj}$. 
Let $\eta_\ell$ be the largest eigenvalue of $\bH_{\ell}$, then we have
\begin{align}\label{logeq2}
	L(\balpha)\le L({\balpha}^\old) +(\balpha_{\ell}-\balpha^\old_{\ell})^\T \nabla L(\balpha^\old)_{\ell}+\frac{\eta_\ell}{2}(\balpha_{\ell}-\balpha^\old_{\ell})^\T(\balpha_{\ell}-\balpha^\old_{\ell}). 
\end{align}
Finally, GMD algorithm to solve \eqref{eq:gradient} iteratively updates $\balpha_{\ell}$ by solving, for $\ell = 0$
\begin{align} \label{logeq3}
	\argmin_{\balpha_{0}} L({\balpha}^\old) +(\balpha_{0}-\balpha_{0}^\old)^\T \nabla L(\balpha^\old)_{0} +\frac{\eta_0}{2}(\balpha_{0}-\balpha_{0}^\old)^\T(\balpha_{0}-\balpha_{0}^\old) +\frac{\lambda\theta_{0}}{2}\balpha_{0}^{\T}\balpha_{0},
\end{align}
and, for $\ell = 1, \cdots, p$
\begin{align} \label{logeq4}
	\argmin_{\balpha_{\ell}} L({\balpha}^\old) +(\balpha_{\ell}-\balpha_{\ell}^\old)^\T \nabla L(\balpha^\old)_{\ell}&+\frac{\eta_\ell}{2}(\balpha_{\ell}-\balpha_{\ell}^\old)^\T(\balpha_{\ell}-\balpha_{\ell}^\old) + \lambda\theta_{\ell}\|\balpha_{\ell}\|_2.
\end{align}
The computational efficiency of the proposed GMD algorithm for \eqref{eq:gradient} comes from the fact that both \eqref{logeq3} and \eqref{logeq4} have a closed-form solution as follows:
\begin{align} 
	\hat{\balpha}_{0} & =(\eta_0+\lambda\theta_0)^{-1}(\eta_0{\balpha}_{0}^\old-\nabla L({\balpha}^\old)_{0}),\label{logsol1} \\
	\hat{\balpha}_{\ell} & =\frac{1}{\eta_\ell}\left(-\nabla L(\balpha^\old)_{\ell}+\eta_\ell\balpha_{\ell}^\old\right)
	\left(1-\frac{\lambda\theta_{\ell}}{\|-\nabla L(\balpha^\old)_{\ell}+\eta_\ell\balpha_{\ell}^\old\|_2}\right)_+, ~ \ell = 1, \cdots, p. \label{logsol2}
\end{align}
Algorithm \ref{ag1} in the following summarizes the proposed GMD algorithm to solve \eqref{eq:gradient}.
\begin{algorithm}[!htbp]
	\caption{GMD algorithm to solve \eqref{eq:gradient}} \label{ag1}
	
	\begin{enumerate}
		\item Compute the largest eigenvalue $\eta_\ell$ of $\bH_{\ell}$, $\ell = 0, 1, \cdots , p$.
		\item Initialize $\balpha$ and set $\balpha = \balpha^\old$.
		\item Repeat the following cyclic groupwise updates until convergence: 
		\begin{enumerate}
			\item[3.1] Compute $\nabla L(\balpha^\old)$ at a current value of $\balpha^\old$.
			\item[3.2] Update 
			\begin{align*}
				\balpha_{0}^\old    & \leftarrow \hat \balpha_{0} ~ \mbox{from \eqref{logsol1}} \\
				\balpha_{\ell}^\old & \leftarrow \hat \balpha_{\ell} ~ \mbox{from \eqref{logsol2}, for $\ell = 1, 2, \cdots, p$}.
			\end{align*}
		\end{enumerate}
	\end{enumerate}
	
\end{algorithm}

Note that each update to $\balpha_{\ell}$'s in \eqref{logsol1} and \eqref{logsol2} for $\ell = 0, 1, \cdots, p$ requires $O(n^2)$ flops to evaluate the sum, which can be inefficient for a large $n$. 
When calculating $L(\balpha^\old)$ in \eqref{logsol1} and \eqref{logsol2}, one can consider only a small number of data points, say $k$, in the nearest neighbor of $\bx_i$ for each $i=1,...,n$. This substantially reduces the computational complexity if $k$ is much less than $n$ and is chosen independently of $n$.

\subsection{GMD algorithm With the Sequential Strong Rule} \label{ss:ssr}
The performance of regularized methods, such as ours, heavily depends on the choice of tuning parameters. 
To efficiently solve \eqref{eq:gradient} using algorithm \ref{ag1} over a grid of tuning parameter values $\lambda_1\ge\lambda_2\ge \cdots \ge \lambda_ k$, we propose to employ the sequential strong rule \citet{tibshirani2012strong}.

To emphasize the dependence on $\lambda$, let us rewrite $\nabla L(\hat{\balpha}(\lambda))_{\ell}$ as $C_{\ell}(\lambda)$, where $\hat \balpha(\lambda)$ denotes the solution at $\lambda$. Assume that $C_{\ell}$ is a Lipschitz function of $\lambda$ with respect to the $L_2$-norm. The sequential strong sequential rule for our problem \eqref{eq:gradient} is to discard $\hat{\balpha}_{\ell}$ at $\lambda=\lambda_{k}$ if 
\begin{align} \label{ssr}
	\|C_{\ell}(\lambda_{k-1})\|_2 <\theta_{\ell}(2\lambda_k-\lambda_{k-1}),
\end{align}
which greatly facilitates the computation of \eqref{eq:gradient} over a grid of $\lambda$, $\lambda_1\ge\lambda_2\ge \cdots \ge \lambda_k$.

In what follows, we justify the sequential strong rule \eqref{ssr} for \eqref{eq:gradient} by showing that \eqref{ssr} implies $\hat{\balpha}_{\ell}(\lambda_k) = \mathbf{0}$. The Krush-Kuhn-Tucker (KKT) stationary conditions for \eqref{eq:gradient} are 
\begin{align}
	\nabla L(\hat{\balpha})_{0}+\lambda\theta_0\hat{\balpha}_{0} = \mathbf{0}, \qquad & \text{and} \nonumber\\
	\nabla L(\hat{\balpha})_{\ell}+\lambda\theta_{\ell}\bbeta_{\ell}= \mathbf{0}, \qquad &\text{for }  \ell=1,\cdots ,p, \label{eq:stationary}
\end{align}
where $\bbeta_{\ell} $ denote the subgradient of $\|\hat{\balpha}_{\ell}\|_2$, i.e, $\bbeta_\ell$ is 
$\hat \balpha_\ell / \|\hat \balpha \|_2$ if $\hat \balpha_\ell \neq \mathbf{0}$, and any vector with 
$\|\bbeta_\ell\|_2 \le 1$ if $\hat{\balpha}_\ell = \mathbf{0}$. Note that \eqref{eq:stationary} implies that $ \hat{\balpha}_{\ell}=\mathbf{0}$ when $\|\bbeta_{\ell}\|_2<1$.  Since $\|\bbeta_\ell\| \le 1$, we have 
\begin{align} \label{eq:unit.bound}
	\|C_{\ell}(\lambda)-C_{\ell}(\lambda')\|_2 
	& = \theta_{\ell} \|\lambda\bbeta_{\ell}(\lambda)-\lambda'\bbeta_{\ell}(\lambda')\|_2 \nonumber \\
	& \le \theta_{\ell}|\lambda-\lambda'|, \qquad 
	\mbox{for any $\lambda, \lambda'$ and $\ell = 1, \cdots, p$}. 
\end{align}
Given $\hat{\balpha}_{\ell}(\lambda_{k-1})$ at $\lambda_{k-1}$, we have from \eqref{ssr} and \eqref{eq:unit.bound} 
\begin{align*}
	\|C_{\ell}(\lambda_k)\|_2 &\le	\|C_{\ell}(\lambda_{k-1})\|_2 + 	\|C_{\ell}(\lambda_k)-C_{\ell}(\lambda_{k-1})\|_2\\
	& < \theta_{\ell}(2\lambda_k-\lambda_{k-1}) +\theta_{\ell}(\lambda_{k-1}-\lambda_{k}) \\
	& = \theta_k \lambda_k,
\end{align*}
which implies $\|\bbeta_{\ell}\|_2<1$ by \eqref{eq:stationary}, and $\hat{\balpha}_{\ell}=\mathbf{0}$. 

The initial $\lambda$ can be set to be $\lambda_{\max}=\underset{\ell=1,\cdots ,p}{\max}\frac{1}{\theta_\ell}\|\nabla L((\hat{\balpha}_{0},\mathbf{0}))_{\ell}\|_2$, the smallest value of the tuning parameter for which $\hat{\balpha}_{(1)},\cdots, \hat{\balpha}_{(p)}$ are exactly zero. In practice, one can estimate the initial $\hat{\balpha}_{0}$ as
\begin{align} \label{eq:alpha0}
	\hat{\balpha}_{0}=\argmin_{{\balpha}_{0}}
	\frac{1}{n^2}	\sumi \sumj  \omega_{s}(\bx_{i}-\bx_{j})L\{y_i \balpha_{0}^{\T} \bk_{j}\delta_{ij0}\} + 
	\lambda_0\frac{\theta_0}{2}\|\balpha_{0}\|_2^2 ,
\end{align}
with an appropriate $\lambda_0$. 

Let $\mathcal{S}(\lambda_k)$ be an index set of predictors survived after the strong rule \eqref{ssr}:
$$\mathcal{S}(\lambda_{k})=\left\{\ell : \|C_{\ell}(\lambda_{k-1})\|_2\ge \theta_{\ell}(2\lambda_{k}-\lambda_{k-1})	 \right\}.$$
After update $\hat{\balpha}(\lambda_k)$ for $\ell \in \mathcal{S}(\lambda_k) $, check the KKT condition \eqref{eq:stationary} $\|\nabla L(\hat{\balpha})_{\ell}\|_2 \le \lambda_{k}\theta_{\ell}$ for all $\ell \notin \mathcal{S}(\lambda_{k})$. Let $\mathcal{V}(\lambda_{k})$ be the indices of the predictors that violate the KKT conditions. If $\mathcal{V}(\lambda_k) =\emptyset$, we are done. Otherwise update $\mathcal{S}(\lambda_k) \leftarrow  \mathcal{S}(\lambda_k)\cup \mathcal{V}(\lambda_k)$ and recompute $\hat{\balpha}(\lambda_k)$ for $\ell \in \mathcal{S}(\lambda_k) $. 
\begin{algorithm}[h]
	\caption{The GMD algorithm with the sequential strong rule}
	\begin{enumerate}
		\item Compute a set $\mathcal{S}(\lambda_k)$
		\item Initialize $\balpha^*=\hat{\balpha}(\lambda_{k-1})$.
		\item For $\ell \in \mathcal{S}(\lambda_k)$, do Step 3 in Algorithm 1 to update $\hat{\balpha}(\lambda_k)_{\ell}$.
		\item Compute a set $\mathcal{V}(\lambda_k)$
		\item If $\mathcal{V}(\lambda_k)=\emptyset$, then return $\hat{\balpha}=\hat{\balpha}(\lambda_k)$;
		\item [] else update $\mathcal{S}(\lambda_k) \leftarrow  \mathcal{S}(\lambda_k)\cup \mathcal{V}(\lambda_k)$ and go to Step 3.
	\end{enumerate}
\end{algorithm}

\section{Asymptotic Analysis} \label{sec4}
In this section, we study the asymptotic properties of the proposed method in terms of both the estimation and the variable selection consistency. Without loss of generality, we assume that only the first $p_0$ predictors are informative for classification. Let $\Rho(\bx)$ and $\rho(\bx)$ be respectively the marginal distribution and density function of $\bX$ whose support is denoted by $\calX$. We also let $D=\max_{\bx,\bu \in \calX}|\bx-\bu|$. 

Now, we assume the following regularity conditions. 
\begin{itemize}
	\item[$A1.$] The support of $\bX$, $\mathcal{X}$ is a non-degenerate compact metric space in $\mathbb{R}^{p},$ and there exists a constatnt $c_{1}$ such that $\sup _{\bx}\left\|\mathbf{H}^{*}(\bx)\right\|_{2} \leq c_{1},$ $\forall \bx \in \mathcal{X}$, where $\mathbf{H}^{*}(\bx)=\nabla^{2} f^{*}(\bx).$ 
	
	\item[$A2.$] For some constants $c_{2} > 0$ and $0< \tau \le 1$, there exists a function $\rho(\bx)$ satisfying
	$$
	|\rho(\bx)-\rho(\mathbf{u})| \leq c_{2} d_{X}(\bx, \mathbf{u})^{\tau}, \text { for any } \bx, \mathbf{u} \in \mathcal{X}
	$$
	where $d_{X}(\cdot, \cdot)$ is the Euclidean distance on $\mathcal{X}$. 
	
	\item[$A3.$] The loss function  $L$ is convex and twice differentiable. Furthermore, 
	there is a universal constant $c_L>0$ such that
	$\{L'(x)\}^2 \le c_L L(x),$ and $\sup_{|t|\le T}L''(t)<\infty$ for some $T>0$. 
	
	\item[$A4.$] For all $\bx, \bx' \in \mathcal{X}$, $K(\bx,\bx')\le 1$, and the equality holds when $\bx=\bx'$. 
	
	\item[$A5.$] There exist some constants $c_{4}$ and $c_{5}$ such that 
	$$
	c_{4} \leq \lim _{n \to \infty} \min _{1 \leq l \leq p} \theta_{l} \leq \lim _{n \to \infty} \max _{1 \leq l \leq p_{0}} \theta_{l} \leq c_{5}
	\quad  \mbox{and} \quad 
	s^{p+2}\lambda  \min _{l>p_{0}} \theta_{l} \rightarrow \infty.
	$$
	
	\item[$A6.$] Let us define 
	$$
	\mathcal{X}_{t}=\left\{\bx \in \mathcal{X}: d_{X}(\mathbf{x}, \partial \mathcal{X})<t\right\}
	$$
	for a constant $t \ge 0$. For any $\ell \leq p_{0},$ there exists a constant $t$ such that 
	$$
	\int_{\mathcal{X} \backslash \mathcal{X}_{t}}\left\|\frac{\partial f^*(\mathbf{x})}{\partial{x_\ell}}\right\|_{2}^2 d \Rho(\mathbf{x})>0. 
	$$
	and, for any $\ell \geq p_{0}+1$
	$$
	\frac{\partial f^*(\mathbf{x})}{\partial x_\ell} \equiv 0, \qquad \forall \mathbf{x} \in \mathcal{X.}
	$$ 
	
\end{itemize}

In Assumption A1, the compactness of the support $\mathcal{X}$ is assumed for technical convenience as conventionally done in the nonparametric models \citep{mukherjee2006estimation, ye2012learning, stefanski2014variable, yang2016model}, in order to ensure that the solution $(\hat{f}, \hat{\bg})$ exists and even is unique. The bounded assumption on $\|\bH^*(\bx)\|_2$ implies $|f^*(\bu)-f^*(\bx)-\nabla f^*(\bx)^\T(\bu-\bx)| \le c_1 \|\bu-\bx\|_2^2$ for any $\bu$ and $\bx$. Assumption A2 states that the marginal density of $\bX$, $\Rho(\bx)$ is H\"{o}lder continuous with the exponent $\tau$. We note that Assumptions A1 and A2 imply that $\Rho(\bx)$ is bounded. 
Assumption A3 limits the variability of the loss function $L(\cdot)$ but is satisfied by widely used loss functions, including logistic loss with $c_L=\frac{1}{2}$ and squared hinge loss with $c_L = \frac{1}{4}.$ Assumption 4 is pretty standard in RKHS theory and holds for popular kernels such as Gaussian kernel $K(\bx, \bx')=\exp(-\|\bx-\bx'\|^2/2\sigma^{2})$, the Laplacian kernel $K(\bx,\bx')=\exp(-\gamma \|\bx-\bx'\|_1)$. 

Under Assumption A4, the norm of the coefficients of a represented function $h(\cdot)=\sum_{i=1}^{n}a_iK(\cdot,\bx_i)$ is less than or equal to the Hilbert norm of $h$, i.e., 
$\|\mathbf{a}\|_2^2 = \sum_{i=1}^{n}a_i^2 \le \|h\|_{\calH_{K}}^2$. 
Assumption A5 states that one should carefully select $\theta_\ell$ to ensure the variable selection consistency of the proposed method. We propose to use $\theta_\ell = \|\nabla \tilde{f}_\ell\|_{\calH_{K}}^{-\gamma}$ for some $\gamma>0$, where $\tilde{f}_{\ell}$ is the minimizer of \eqref{eq:large.margin} with $J(f)=\|f\|_{\calH_{K}}^2$ and $\nabla \tilde{f}_\ell = \frac{\partial\tilde{f}}{\partial\bx_\ell}$. 
Assumption A6 requires that the true gradient function with respect to the informative variable  $\frac{\partial f^*}{\partial{x_\ell}}$ for $\ell \le p_0$ is  substantially different from zero and  is zero for $\ell \ge p_0+1$. 

\begin{theorem}\label{thm1} Let $(\hat{f},\hat{\bg})$ be the sample estimator of $(f^*, \nabla f^*)$ obtained from \eqref{eq:gradient}. Suppose that A1-A6 are satisfied. For $0<\eta<1/2$, $0<s<\lambda\le 1$, there exists a constant $C$ such that with probability greater than $1-\eta$
	\begin{align} \label{thm1form}
		\max \left\{\|\hat{f}-f^*\|^{2},\|\hat{\bg}-\nabla f^*\|^{2}\right\} \le C\left\{R^{2} s^{\tau}+Rs^{-\tau}\left(\mathscr{E}(s,\lambda)\right)\right\},
	\end{align}
	where 
	\begin{align*}
		\mathscr{E}(s,\lambda)=&\left[\left(1+ \frac{c_L}{2\theta_0\lambda n}+\frac{Dps^{\frac{p+2}{2}}}{\sqrt{n}}\right) \left(\frac{L_{R}R+M_{R}\log{\frac{2}{\eta}}}{\sqrt{n}{s^{p+2}}}+{s^2}+{\lambda }
		\right) + \frac{c_L}{2\theta_0\lambda n} +   \frac{(c_L+1)Dp}{ s^{\frac{p+2}{2}}\sqrt{n} }\right].
	\end{align*}
	and
	\begin{align*}
		R= c_R&\left[\left(1+ \frac{c_L+\sqrt{c_L}Dp}{ ns^{\frac{p+2}{2}}}\right)\left(\left(\frac{L_R}{\sqrt{\lambda s^{p+2}}} +M_R\log\frac{2}{\eta} \right)\frac{1}{\sqrt{n}\lambda s^{p+2}}+\frac{s^2}{\lambda }+1
		\right) + \frac{c_L}{2\theta_0\lambda n} +   \frac{(c_L+1)Dp}{ s^{\frac{p+2}{2}}\sqrt{n} }\right]^{1/2}
	\end{align*}
	for some positive constants $c_R, L_R$, and $M_R$ are explicitly given in Appendix.
\end{theorem}

Theorem \ref{thm1} establishes the consistency of both $\hat{f}$ and $\hat{\bg}$ under proper choices of $\lambda$ and $s$. The following corollaries provide a type of $L_2$ bound of estimators $(\hat f, \hat \bg)$, which is essential to establish the selection consistency of the proposed method. 
\begin{corollary}
	Let $L$ be the logistic loss. Choose $s=n^{-\frac{1}{3(p+2+2\tau) } } $ and $\lambda=s^{2\tau}$. There then exists a constant $C>0$ such that for any $0<\eta<1$ with confidence $1-2\eta$, 
	$$	\max \left\{\|\hat{f}-f^*\|^{2},\|\hat{\bg}-\nabla f^*\|^{2}\right\} \le C\log\frac{2}{\eta} n^{-\frac{\tau}{3(p+2+2\tau)}}. $$	
\end{corollary}
\begin{corollary}
	Let $L$ be the squared hinge loss. Choose $s=n^{-\frac{1}{4(p+2+2\tau) } } $ and $\lambda=s^{2\tau}$. There then exists a constant $C>0$ such that for any $0<\eta<1$ with confidence $1-2\eta$, 
	$$	\max \left\{\|\hat{f}-f^*\|^{2},\|\hat{\bg}-\nabla f^*\|^{2}\right\} \le C\log\frac{2}{\eta} n^{-\frac{\tau}{4(p+2+2\tau)}}. $$	
\end{corollary}
For ease of representation, let $\mathcal{A}^*=\left\{1, \cdots , p_0\right\}$ be the index sets that contain all indices of true informative variables, and $\hat{\mathcal{A}}= \left\{\ell : \|\hat{\balpha}_{\ell}\|_2 >0, j=1, \cdots , p \right\}$ be the corresponding estimator from \eqref{eq:gradient}. Finally, Theorem \ref{thm2} states the variable selection consistency of our method.

\begin{theorem}\label{thm2}
	Suppose that all the assumptions in Theorem 1 are satisfied. Choose $\lambda$ and $s$ as Corollary 1 and 2 depending on the loss. Then
	$ P\left(\widehat{\mathcal{A}}=\mathcal{A}^{*}\right) \rightarrow 1 $ as $n \to \infty$.
\end{theorem}

\section{Simulation} \label{sec5}
\input{figure/fig1.tex}
We conducted a simulation study to evaluate the finite-sample performance of the proposed method. We employ the Gaussian kernel $K(\bx,\bu)=e^{-\|\bx-\bu\|_2^2/2\sigma^2}$. The scalar parameters $\sigma^2$ and $s^2$ in $\omega_{s}(\bx-\bu)$ are set to be the median over the pairwise distances between all the sample points as suggested by \citet{mukherjee2006estimation}. Two different loss functions, the logistic and the squared hinge loss (resp. denoted by Logit and Hinge$^2$) are employed. 

As competing methods, we consider sparse kernel discriminant analysis (SKDA; \citealp{stefanski2014variable}), random forest variable selection through backward elimination (RF; \citealp{diaz2006gene}), the component selection and smoothing (COSSO ;\citealp{lin2006component}), and linear logistic regression with SCAD penalty (SCAD). The tuning parameters in all methods are selected by ten-fold cross-validation to minimize the classification error. For the random forest, we set the number of trees as 3000 and the proportion of variables to be eliminated at each iteration as 10\%.

In addition to the variable selection performance, we also compare the prediction performance. Although our method can estimate the classifier, we do not directly use $\hat f$ for the prediction since it cannot fully reflect the variable selection effect. Instead, we compute the test errors of our methods by fitting ordinary logistic or squared hinge classifiers \eqref{eq:large.margin} using selected variables only.  

To simulate the data, we assume the following model 
$$
y_i = \mbox{sign} \{f(\bx_i) + 0.2\epsilon_i\}
$$
where $\epsilon_i \sim N(0, 1)$ with four different classification functions:
\begin{itemize}
	\item [] (M1) $f(\bx) = x_{i1} - x_{i2}$
	\item [] (M2) $f(\bx) = \sqrt{x_{i1}^2+ x_{i2}^2}\log\sqrt{x_{i1}^2+ x_{i2}^2}$
	\item [] (M3) $f(\bx) = x_{i1}^2 - x_{i2}^2- 0.25 $
	\item [] (M4) $f(\bx) = x_{i1}x_{i2}$
\end{itemize}
The $p$-dimensional predictor $\bx_i =(x_{i1},\cdots , x_{ip})^\T$ are generated as follows. We first generate $W_{ij}$ and $U_i$ from $U(-2,2)$ independently where $i=1,\cdots ,n$ and $j=1,\cdots ,p$ $x_{ij}=\frac{1}{2}(W_{ij}+U_{i})$. We set $n = 500$ and $p \in \{10, 50\}$. Under all models (M1)--(M4), the first two predictors are informative, i.e., $\mathcal{A} = \{1, 2\}$. 

The first model (M1) corresponds to a simple linear classification problem, while the rest (M2)--(M4) are highly nonlinear. The classification boundary of (M2) is a circle with radius 1, and that of (M3) is a hyperbola with the foci $(-1/\sqrt{2},0)$ and $(1/\sqrt{2},0)$.  The last model (M4) represents the interaction model between  $x_1$ and $x_2$. See Figure \ref{fig:sim}, which illustrates the classification boundaries of all models.

\begin{table*}[!htbp]
	\centering
	\caption{The averaged performance measures of various variable selection methods in (M1)--(M4).  The numbers in parentheses are the corresponding standard deviations.}\label{tb1}
	\scalebox{.8}{
		\begin{tabular}{@{}cc rrrr c rrrr@{}}
			\toprule
			\multirow{2}{*}{Model}&   \multirow{2}{*}{Method}        & \multicolumn{4}{c}{$p=10$}                       &  & \multicolumn{4}{c}{$p=50$}                         \\ \cmidrule{3-6} \cmidrule{8-11}
			&     & \multicolumn{1}{c}{TP} & \multicolumn{1}{c}{FP} & \multicolumn{1}{c}{Correct} & \multicolumn{1}{c}{Test error}
			&     & \multicolumn{1}{c}{TP} & \multicolumn{1}{c}{FP} & \multicolumn{1}{c}{Correct} & \multicolumn{1}{c}{Test error}\\ \midrule
			\multirow{6}{*}{M1} 
			& Logit     & 2.00 (0.00)   & 0.29 (0.69) &  81$~~$ & 0.079 (0.012) &  & 2.00 (0.00) &  0.55 (1.36)   & 82$~~$ & 0.079 (0.014) \\
			& Hinge$^2$ & 2.00 (0.00)   & 0.71 (1.27) &  71$~~$ & 0.080 (0.012) &  & 2.00 (0.00) &  0.94 (1.77)   & 69$~~$ & 0.081 (0.013) \\
			& SKDA      & 2.00 (0.00)   & 0.10 (0.30) &  90$~~$ & 0.087 (0.013) &  & 2.00 (0.00) &  0.07 (0.29)   & 94$~~$ & 0.085 (0.013) \\
			& RF        & 2.00 (0.00)   & 2.56 (3.20) &  53$~~$ & 0.100 (0.014) &  & 2.00 (0.00) & 29.41 (19.3)   & 15$~~$ & 0.124 (0.024) \\
			& COSSO     & 2.00 (0.00)   & 0.00 (0.00) & 100$~~$ & 0.082 (0.013) &  & 2.00 (0.00) &  0.02 (0.20)   & 99$~~$ & 0.090 (0.049) \\
			& SCAD      & 2.00 (0.00)   & 0.36 (0.87) &  81$~~$ & 0.078 (0.011) &  & 2.00 (0.00) &  0.77 (1.67)   & 75$~~$ & 0.076 (0.012) \\
			\midrule
			\multirow{6}{*}{M2} 
			& Logit     & 2.00 (0.00) & 0.00 (0.00) & 100$~~$ & 0.158 (0.018) &  & 2.00 (0.00) &  0.11 (0.31)   & 89$~~$ & 0.162 (0.020) \\
			& Hinge$^2$ & 2.00 (0.00) & 0.00 (0.00) & 100$~~$ & 0.152 (0.017) &  & 2.00 (0.00) &  0.10 (0.33)   & 91$~~$ & 0.155 (0.020) \\
			& SKDA      & 2.00 (0.00) & 0.72 (0.45) &  28$~~$ & 0.137 (0.018) &  & 2.00 (0.00) &  0.71 (0.48)   & 30$~~$ & 0.138 (0.016) \\
			& RF        & 2.00 (0.00) & 1.41 (2.37) &  64$~~$ & 0.145 (0.017) &  & 2.00 (0.00) & 10.59 (15.2)   & 40$~~$ & 0.155 (0.021) \\
			& COSSO     & 1.99 (0.10) & 0.00 (0.00) &  99$~~$ & 0.133 (0.021) &  & 1.66 (0.57) &  4.95 (4.09)   &  5$~~$ & 0.186 (0.087) \\
			& SCAD      & 0.18 (0.41) & 0.68 (1.21) &   0$~~$ & 0.518 (0.023) &  & 0.02 (0.14) &  0.81 (1.50)   &  0$~~$ & 0.520 (0.026) \\
			\midrule
			\multirow{6}{*}{M3} 
			& Logit     & 2.00 (0.00) & 0.01 (0.10) & 99$~~$ & 0.093 (0.016) &  & 2.00 (0.00) &  0.00 (0.00)   & 100$~~$ & 0.096 (0.015) \\
			& Hinge$^2$ & 2.00 (0.00) & 0.01 (0.10) & 99$~~$ & 0.092 (0.015) &  & 2.00 (0.00) &  0.00 (0.00)   & 100$~~$ & 0.094 (0.015) \\
			& SKDA      & 2.00 (0.00) & 0.04 (0.20) & 96$~~$ & 0.102 (0.017) &  & 2.00 (0.00) &  0.04 (0.20)   &  96$~~$ & 0.105 (0.017) \\
			& RF        & 2.00 (0.00) & 1.46 (2.49) & 63$~~$ & 0.106 (0.015) &  & 2.00 (0.00) & 20.97 (19.3)   &  23$~~$ & 0.133 (0.030) \\
			& COSSO     & 1.00 (0.00) & 0.00 (0.00) &  0$~~$ & 0.198 (0.019) &  & 0.90 (0.44) &  5.43 (3.71)   &   0$~~$ & 0.226 (0.065) \\
			& SCAD      & 0.13 (0.39) & 0.63 (1.36) &  0$~~$ & 0.359 (0.022) &  & 0.06 (0.28) &  1.29 (2.71)   &   0$~~$ & 0.361 (0.024) \\ \midrule
			\multirow{6}{*}{M4} 
			& Logit     & 2.00 (0.00) & 0.01 (0.10) &  99$~~$ & 0.203 (0.021) &  & 2.00 (0.00) &  0.16 (0.47)   & 88$~~$  & 0.205 (0.024) \\
			& Hinge$^2$ & 2.00 (0.00) & 0.00 (0.00) & 100$~~$ & 0.197 (0.020) &  & 2.00 (0.00) &  0.34 (1.30)    & 86$~~$  & 0.202 (0.031) \\
			& SKDA      & 1.98 (0.20) & 0.52 (0.63) &  52$~~$ & 0.219 (0.029) &  & 2.00 (0.00) &  0.85 (0.43)   & 18$~~$  & 0.213 (0.022) \\
			& RF        & 2.00 (0.00) & 2.17 (2.83) &  50$~~$ & 0.200 (0.020) &  & 2.00 (0.00) & 15.18 (14.4)   & 21$~~$  & 0.241 (0.045) \\
			& COSSO     & 1.93 (0.26) & 0.35 (0.73) &  71$~~$ & 0.366 (0.072) &  & 1.30 (0.68) &  9.70 (6.37)   &  0$~~$  & 0.379 (0.086) \\
			& SCAD      & 0.17 (0.47) & 0.73 (1.56) &   0$~~$ & 0.635 (0.080) &  & 0.02 (0.20) &  1.12 (2.49)   &  0$~~$  & 0.639 (0.073) \\	
			
			\bottomrule
		\end{tabular}
	}
\end{table*}
We report both true positives (TP) and false positives (FP) that count the numbers of correctly selected variables and incorrectly selected variables, respectively, to evaluate the variable selection performance. The larger TP and the smaller FP represent the better model, and the perfect variable selection performance corresponds to the TP of 2 and the FP of 0. We also count the number of cases in which a variable selection method estimates the true model with $x_1$ and $x_2$ only (Correct). Finally, the classification error rates for the independent test set with size $n=1000$ are reported to compare the prediction accuracy after the variable selection. 

Table \ref{tb1} contains the aforementioned performance measures averaged over 100 independent repetitions. As expected, in the simplest linear case (M1), the logistic regression with the SCAD penalty shows promising performance. It is rather surprising that COSSO shows almost perfect performance under (M1). All other methods, except RF, perform reasonably well in terms of both variable selection and prediction accuracy. RF always tends to produce larger numbers of TP than others since it employs the backward elimination approach for variable selection.

The logistic regression with the SCAD penalty completely fails under the nonlinear scenarios (M2)--(M4). COSSO is still competitive when $p = 10$ while it deteriorates when $p = 50$. Although SKDA performs reasonably well in most cases, the proposed method (regardless of loss functions) outperforms in these nonlinear scenarios. We would like to point out that our methods show much better performance, especially when $p = 50$.

\section{Illustration to Real Data} \label{sec6}

As a final showcase, we apply the proposed methods to Wisconsin breast cancer data(WBCD). The WBCD data is available from the UC Irvine Machine Learning Repository (\hyperbaseurl{http://archive.ics.uci.edu/ml}http://archive.ics.uci.edu/ml). The data contains information from 569 ($=n$) patients, including binary indicators of tumor malignancy (malicious/severe) and 30($=p$) possibly related predictors. We standardized each predictor marginally to have zero-mean and unit-variance. In this section, we report the results of the logistic loss only to avoid redundancy since the squared hinge loss functions showed nearly identical results.

\input{figure/fig2.tex}

We randomly split into a training set of 300 observations and a test set of the remaining 269. Ten-fold cross-validation is used to choose the optimal tuning parameter. Figure \ref{fig:real} illustrates the 10-fold cross-validated error rate with respect to the grid of tuning parameter $\lambda$, efficiently computed by the strong sequential rule proposed in Section \ref{ss:ssr}. As shown in Figure \ref{fig:real} $X_{22}, X_{21}$ and $X_{25}$ are selected by our method. 

We also compare the results for different methods considered in Section \ref{sec5}, except COSSO, which does not converge for this particular example. Table \ref{tb:wdbc1} shows the number of selected variables and the test error averaged over forty random splitting replications. One can conclude that the proposed method outperforms all the others because it achieves the smallest test error with the second smallest number of selected variables. SKDA selects the smallest number of variables but yields the worst performance. Finally, we take a closer look at the variable selection result. Table \ref{tb:wdbc2} contains the frequency of each variable selected over forty independent random splittings. The proposed method always selects $X_{21}$ and $X_{22}$ that seem to be informative since these two variables are frequently selected by other methods as well. $X_{28}$ also looks informative for the same reason. The selection pattern between our method and SKDA is generally similar, but $X_{25}$ makes a difference between our method and SKDA.  

\begin{table}[]
	\centering
	\caption{Average test errors  and the number of variables selected for the WBCD. The numbers in parentheses are the corresponding standard deviations. }
	\label{tb:wdbc1}
	\scalebox{0.9}{
		\begin{tabular}{@{}cccccc@{}}
			\toprule
			Methods    &     Logit      &      SKDA      &       RF       &      SCAD      \\ \midrule\midrule
			Test error   & 0.0401(0.0125) & 0.0603(0.0199) & 0.0552(0.0148)  & 0.0454(0.0153) \\
			No. variables &   4.33(1.72)   &   2.93(0.57)   &   11.3(5.60)   &   5.08(1.16)   \\ \bottomrule
		\end{tabular}
	}
\end{table}

\begin{table}[]
	\centering
	\caption{Results of the selection frequency for all predictors in the WBCD data.}
	\label{tb:wdbc2}
	\scalebox{0.9}{
		\begin{tabular}{@{}ccccccccccc@{}}
			\toprule
			$X_{j}$  & Logit & SKDA & RF &  SCAD & $X_{j}$  & Logit & SKDA & RF & SCAD \\ \midrule\midrule
			$X_{1}$  &   1   &  4   & 25 & 0   & $X_{16}$ &   2   &  0   & 1  &  0  \\
			$X_{2}$  &   1   &  6   & 5  & 11  & $X_{17}$ &   0   &  0   & 2  &  0  \\
			$X_{3}$  &   0   &  0   & 26 & 0   & $X_{18}$ &   0   &  0   & 2  &  3  \\
			$X_{4}$  &   0   &  0   & 29 & 0   & $X_{19}$ &   1   &  0   & 1  &  2  \\
			$X_{5}$  &   0   &  0   & 2  & 2   & $X_{20}$ &   1   &  0   & 1  &  1  \\
			$X_{6}$  &   0   &  0   & 7  & 0   & $X_{21}$ &  40   &  35  & 38 & 40  \\
			$X_{7}$  &   0   &  7   & 28 & 8   & $X_{22}$ &  40   &  27  & 18 & 36  \\
			$X_{8}$  &   4   &  9   & 37 & 3   & $X_{23}$ &   2   &  0   & 39 &  0  \\
			$X_{9}$  &   0   &  0   & 1  & 3   & $X_{24}$ &   1   &  0   & 40 &  0  \\
			$X_{10}$ &   2   &  1   & 1  & 0   & $X_{25}$ &  31   &  2   & 6  & 21  \\
			$X_{11}$ &   7   &  1   & 14 & 19  & $X_{26}$ &   0   &  0   & 13 &  0  \\
			$X_{12}$ &   0   &  0   & 1  & 0   & $X_{27}$ &   9   &  0   & 31 &  5  \\
			$X_{13}$ &   0   &  0   & 14 & 1   & $X_{28}$ &  25   &  22  & 40 & 30  \\
			$X_{14}$ &   0   &  0   & 24 & 0   & $X_{29}$ &   6   &  2   & 3  & 14  \\
			$X_{15}$ &   0   &  0   & 1  & 4   & $X_{30}$ &   0   &  1   & 2  &  0  \\ \bottomrule
		\end{tabular}
	}
	
\end{table}

\section{Conclusion} \label{sec7}
In this article, we develop a new nonparametric variable selection method in binary classification. The proposed method is based on learning gradients, which should be zero for non-informative predictors. We develop an efficient algorithm to compute the gradient estimator and establish its variable selection consistency. Both the simulation and the real data example illustrate the promising performance of the proposed method. 

It is noteworthy that our theoretical results are valid only for twice-differentiable loss functions, which excludes some popular loss functions, such as hinge loss and the large-margin unified machine loss \citep{liu2011hard}. This is because the proof relies on the second-order Taylor expansion of the objective function to obtain its risk bound. However, one can extend the results by exploiting alternative tools other than the Taylor expansion. 

The gradient functions of the models (i.e., regression or classification functions) contain useful information for to understand the data generating process, not included in the original function. For example, shape constraints such as monotonicity can be imposed by properly controlling gradients. The gradient of the models can also be used to achieve  dimension reduction of predictors \citep{xia2009adaptive}. This warrants further investigation on gradient learning.

\bibliographystyle{spbasic}      
\bibliography{references}
\section*{Appendix}
\input{proof.tex}
\end{document}

%% file: figure/fig1.tex
\begin{figure*}[t]
\centering
\subfigure[M1 - $f(\bx) = x_{i1} - x_{i2}$]{
\includegraphics[width = 0.4\textwidth]{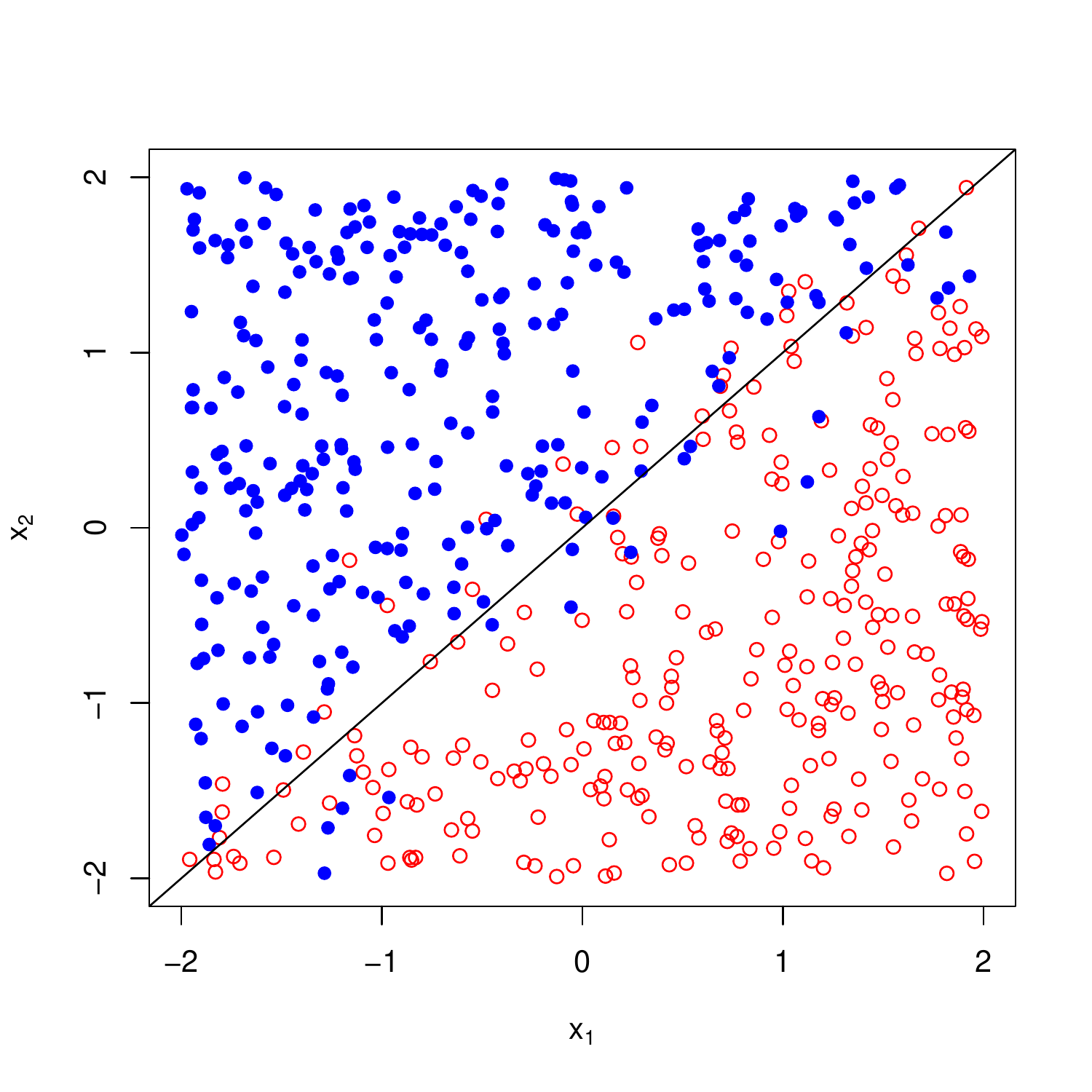}}
\subfigure[M2 - $f(\bx) = \sqrt{x_{i1}^2+ x_{i2}^2}\log\sqrt{x_{i1}^2+ x_{i2}^2}$]{
\includegraphics[width = 0.4\textwidth]{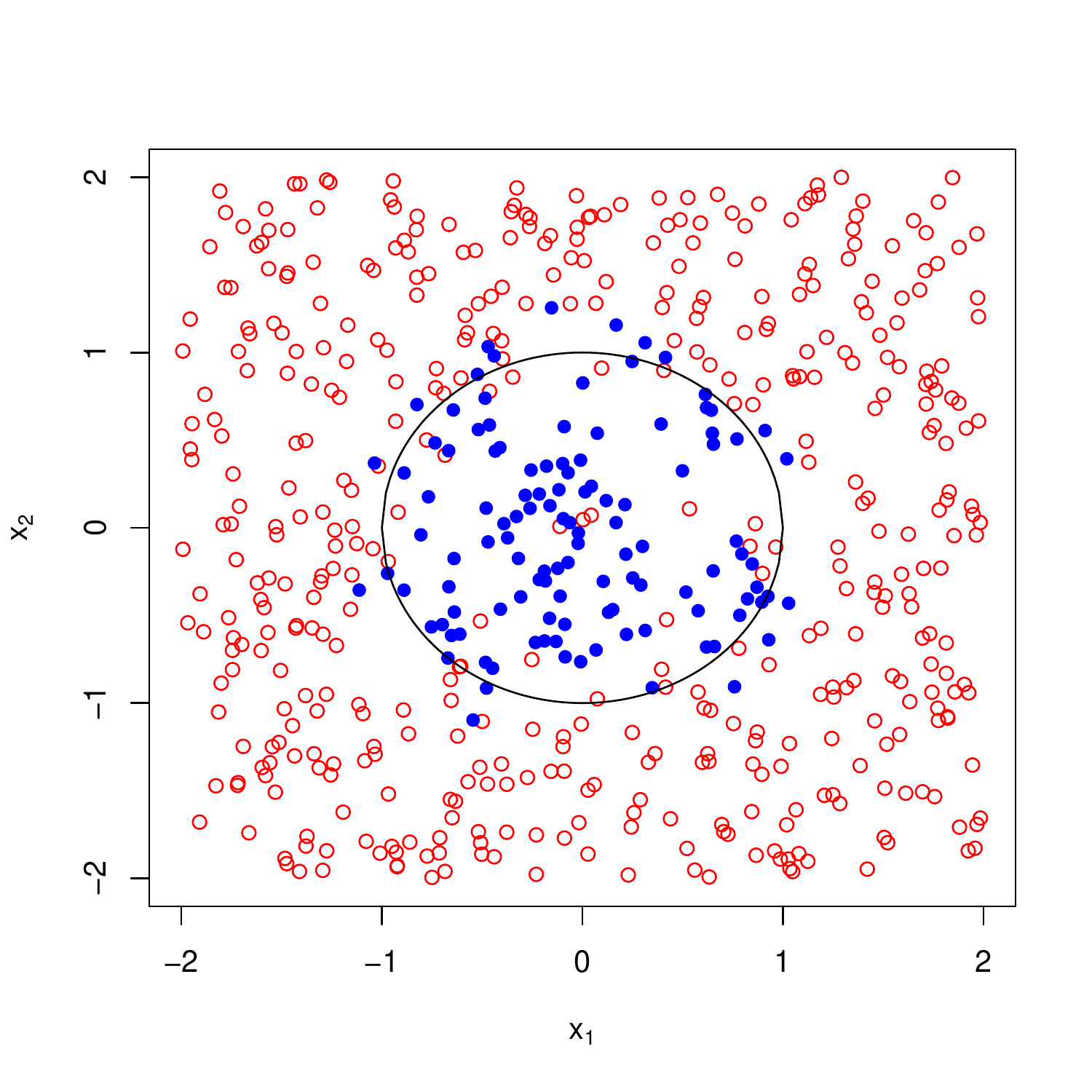}}\\
\subfigure[M3 - $f(\bx) = x_{i1}^2 - x_{i2}^2- 0.25 $]{
	\includegraphics[width = 0.4\textwidth]{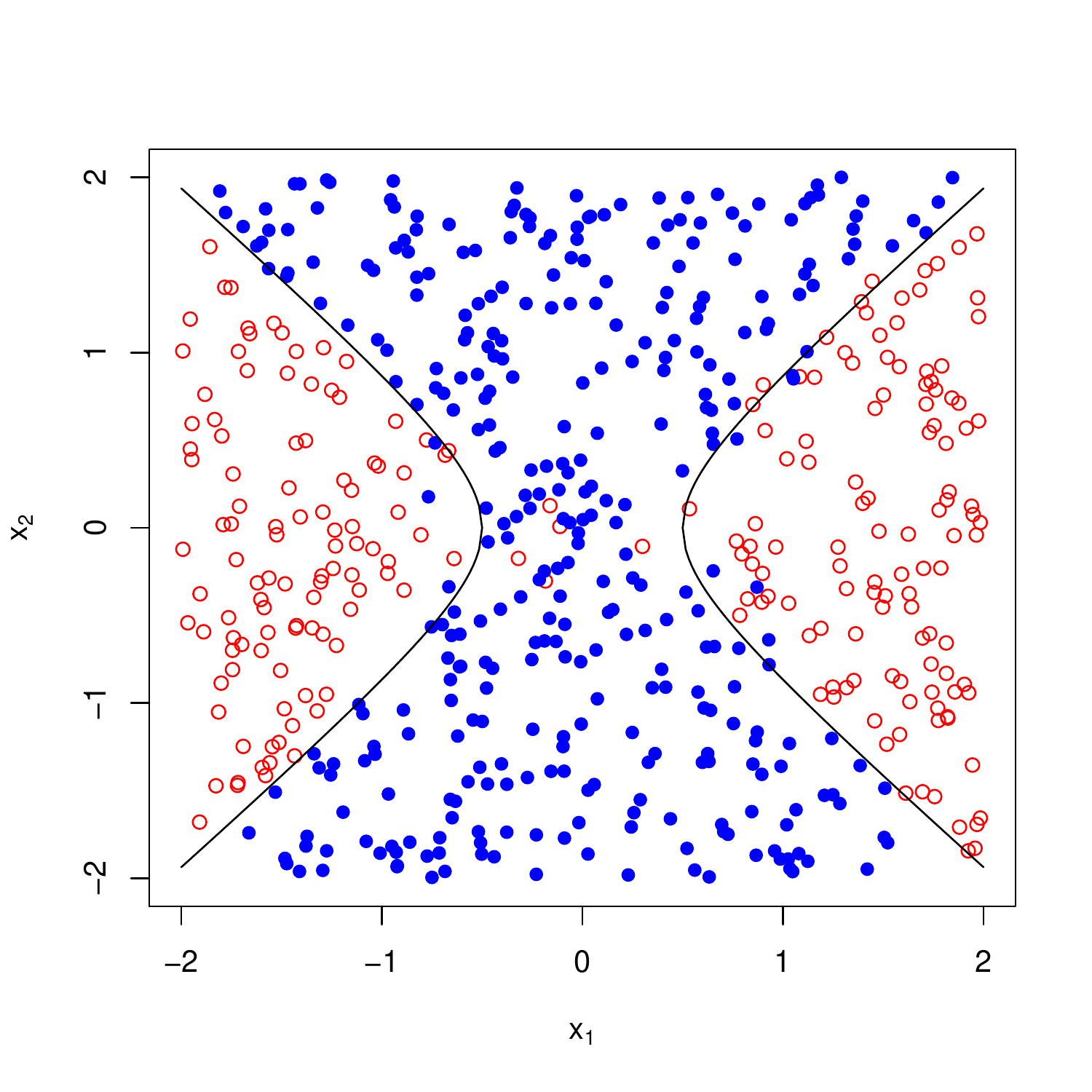}}
\subfigure[M4 - $f(\bx) = x_{i1}x_{i2}$]{
	\includegraphics[width = 0.4\textwidth]{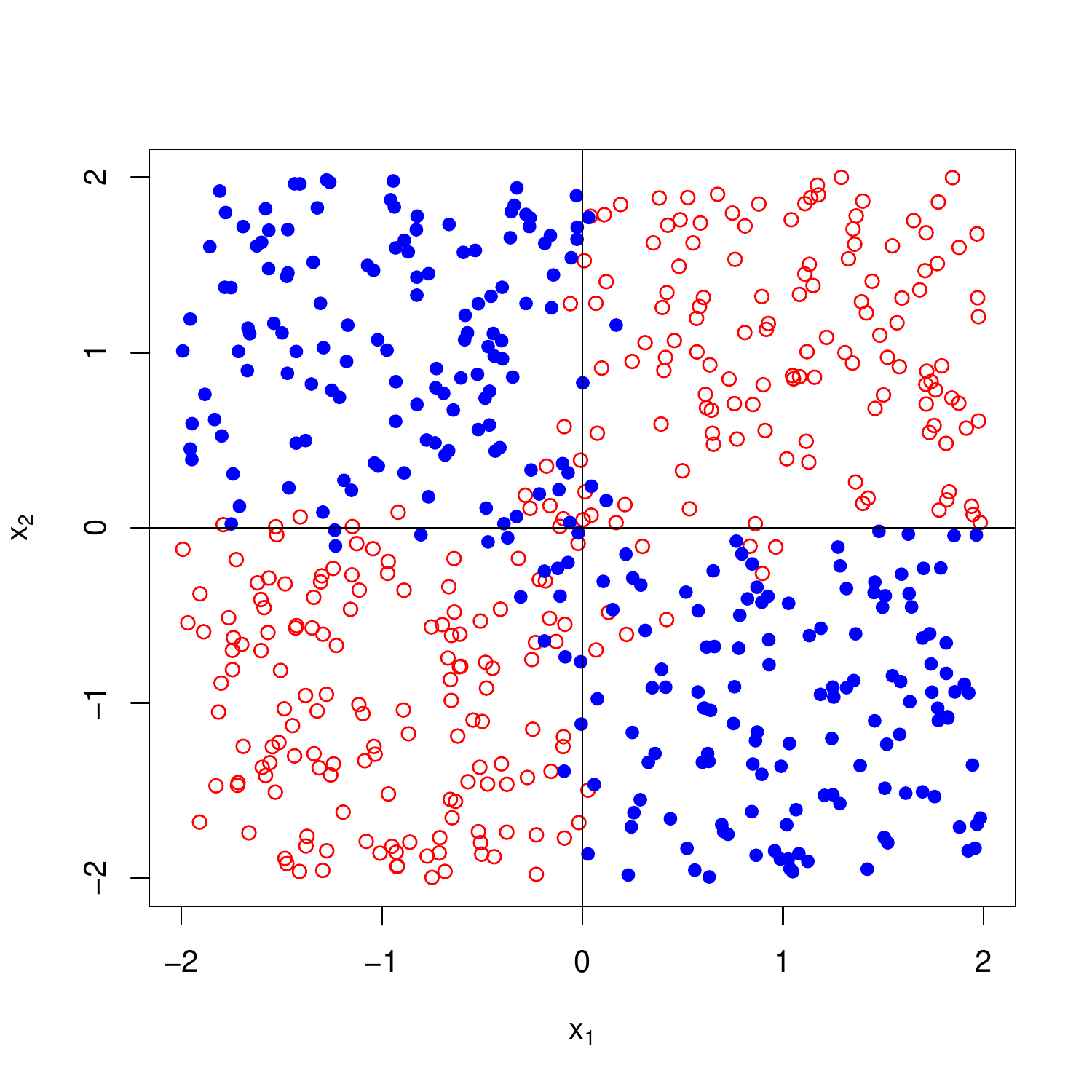}}
\caption{Illustration of the classification functions considered in the simulation. }\label{fig:sim}
\end{figure*}

%% file: figure/fig2.tex
\begin{figure}[t]
\centering
\includegraphics[width = 0.7\textwidth]{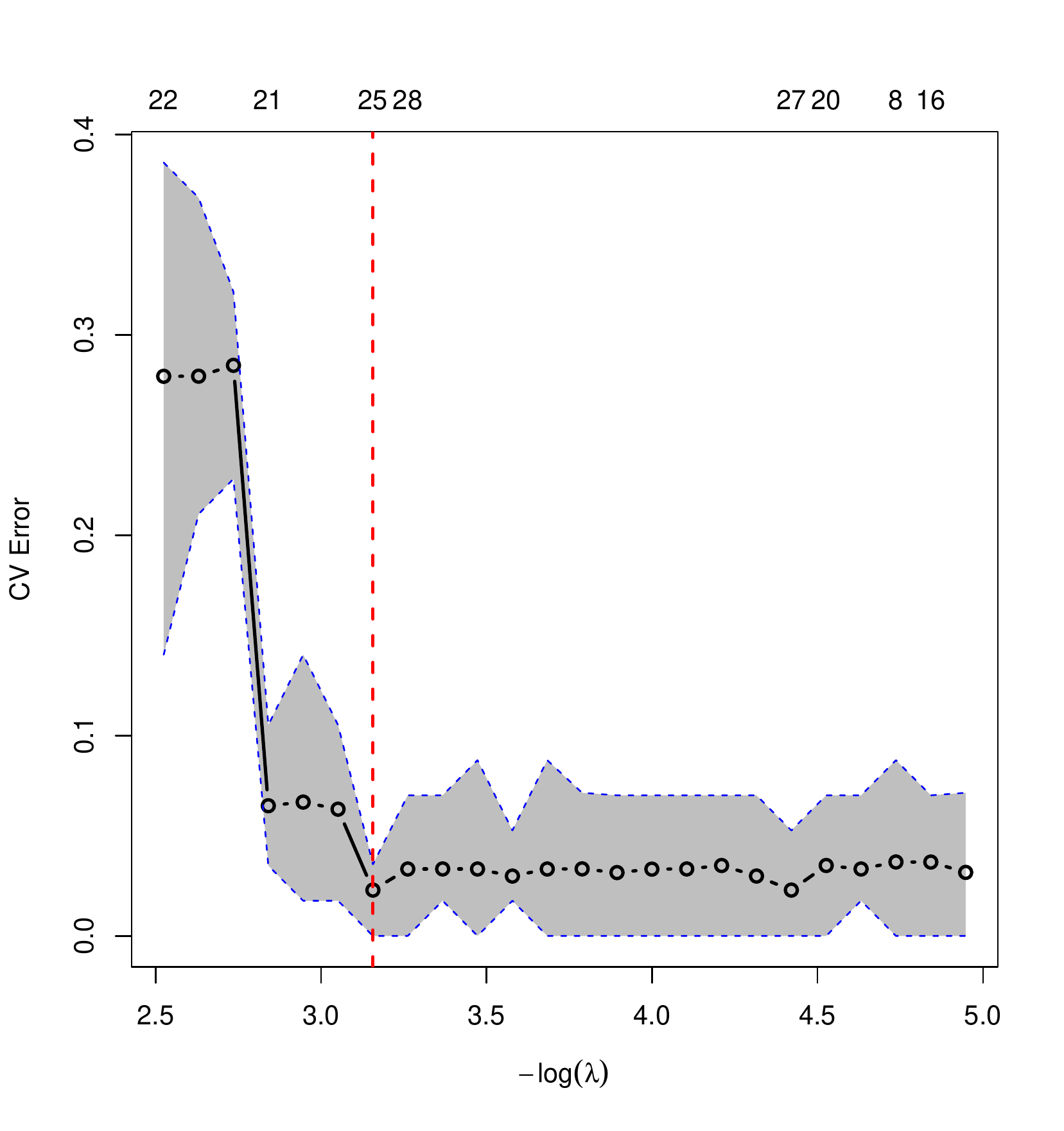}\\
\caption{A cross-validation error rate for different values of $\lambda$. The vertical lines represent the selected $\lambda$ using the strong sequential rule described in Section 3.2. The number at the top of the graph represents the index of the predictor that was first selected for that tuning parameter value. The gray area represents the range of the cross-validated error.}\label{fig:real}
\end{figure}

%% file: proof.tex
\subsection{Proof of Theorem  \ref{thm1}}
Let $(\hat{f},\hat{\bg})$ be our estimator calculated from \eqref{eq:gradient} and $f^*$ be the true classification function defined in \eqref{eq:loss}. Similar to the empirical error in \eqref{eq:def1}, define the expected error
\begin{align} \label{emp}
	\mathcal{E}(\hat{f},\hat{\bg})=\int\int
	\omega_{s}(\bx-\bu)L\left(
	y\left(
	\hat{f}(\bu)+ \hat{\bg}(\bu)(\bx-\bu)\right)\right)d{\Rho}(\bu)d\Rho(\bx,y).
\end{align}
Since
 $$E[\hat{\calE}(\hat{f},\hat{\bg})]=\frac{1}{ns^{p+2}}R(f^*)+\frac{n-1}{n}\mathcal{E}(\hat{f},\hat{\bg}),$$
 the expected value of the empirical error gets close to the expected error if the empirical error is concentrated as $n$ increases. Denote
\begin{align} \label{R1} 
	\mathcal{R}_s&=\int\int\omega_{s}(\bx-\bu)L\left(yf^*(\bx)\right)d{\Rho}(\bx,y)d{\Rho}(\bu),
\end{align}
a weighted version of \eqref{eq:err_loss}. Then the excess error can be defined by $\mathcal{E}(\hat{f},\hat{\bg})-\mathcal{R}_s$.
 The idea behind the proof for Theorem \ref{thm1} is to first bound the $L_{\rho_\bx}^2$ difference by the excess error and then bound the excess error.  Our minimization problem \eqref{eq:gradient} is quite similar to \citet{mukherjee2006estimation} except that it deals with different type of regularization. Typically, we consider the coefficient-based space $\calH_{\alpha}=\left\{f: f(\cdot)=\sum_{i=1}^{n}\alpha_i K(\bx_i,\cdot), \alpha_i \in \mathbb{R} \right\}$ as the candidate functional space, which depends on $\{\bx_i\}_{i=1}^n$. One difficulty is that we cannot define $J_0(f^*)$ and $J_1(\bg^*)$ because $f^*$ and $g_1^*, ..., g_p^*$ may not be included in $\calH_{\alpha}$. To avoid this difficulty, we introduce an intermediate learning algorithm as a bridge for error analysis to use standard empirical processes and approximation theory.
Denote $\calH_K^{p+1}=\left\{(f, g_1, ..., g_p): f \in \calH_{K}, g_j \in \calH_{K} \right\}$ and $\calH_{\alpha}^{p}=\left\{(g_1, ..., g_p):  g_j \in \calH_{\alpha}\right\}$.

We introduce intermediate learning algorithms to establish consistency.

\begin{align} \label{int}
	(\bar{f}, \bar{\bg})= \argmin_{(f,\bg)\in \calH_{K}^{p+1}} \hat{\calE}(\bar{f},\bar{\bg})+\frac{\zeta}{2}\left(\theta_0\|f\|_{\calH_K}^2+\sum_{\ell=1}^p \theta_{\ell}\|g_{\ell}\|_{\calH_K}^2\right),
\end{align}
where $\zeta>0$ is a tuning parameter. 
In \eqref{int}, we apply adaptive tuning parameters to the penalty term of the original gradient learning in \citet{mukherjee2006estimation}.
By the representer theorem, $\bar{f}$ and $\bar{g}_1$, ..., $\bar{g}_p$ in \eqref{int} have closed solution with the form
$$
\bar{f}(\bx)=\sum_{i=1}^{n}\bar{\alpha}_{i0}K(\bx,\bx_i), \quad \bar{g}_{\ell}(\bx_{i})=\sum_{i=1}^{n}\bar{\alpha}_{i\ell}K(\bx,\bx_i), \qquad \text{ for } \ell=1, ..., p. 
$$

	\begin{lemma}\label{lem_pen} Under A3, 
	\begin{align}
		J_0(\bar{f}_0) \le \frac{c_L}{2\zeta^2\theta_0n}\hat{\calE}(\bar{f},\bar{\bg}) , \quad 		J_1(\bar{g}_\ell) \le \frac{\sqrt{c_L} D}{\zeta\theta_\ell \sqrt{n}} \sqrt{\hat{\calE}(\bar{f},\bar{\bg})}, \qquad \ell=1, ..., p.
	\end{align}

	\end{lemma}
\begin{proof}
	
Taking derivative with respect to $\bar{\balpha}_{\ell}=(\bar\alpha_{1\ell }, ..., \bar\alpha_{n\ell})^\T$ to \eqref{int}, we have
$$\zeta\theta_{\ell}\bK\bar{\balpha}_{\ell}+ \frac{1}{n^2}\sum_{i=1}^{n}\sum_{j=1}^{n} \omega_{s}(\bx_i-\bx_j) L'\left(y_{i}\left(\bar{f}\left(\bx_{j}\right)+\bar{\bg}\left(\bx_{j}\right)^\T \left(\bx_{i}-\bx_{j}\right)\right)\right)\bk_{j}(x_{i\ell}-x_{j\ell})=0.$$

Without loss of generality, we assume that $\bK$ is invertible. Then we can solve for $\bar{\alpha}_{j\ell}$ as 
\begin{align}
\zeta\theta_{\ell}\bar{\alpha}_{j\ell}=- \frac{1}{n^2}\sum_{i=1}^{n} \omega_{s}(\bx_i-\bx_j) L'\left(y_{i}\left(\bar{f}\left(\bx_{j}\right)+\bar{\bg}\left(\bx_{j}\right)^\T \left(\bx_{i}-\bx_{j}\right)\right)\right)(x_{i\ell}-x_{j\ell}).	
\end{align}

Assuming $A3$, then by the H\"{o}lder inequality, we have
\begin{align*}
\zeta^2\theta_{\ell}^2\|\bar{\balpha}_{\ell}\|_2^2&= \sum_{j=1}^n\left[-\frac{1}{n^2}\sum_{i=1}^{n} \omega_{s}(\bx_i-\bx_j) L'\left(y_{i}\left(\bar{f}\left(\bx_{j}\right)+\bar{\bg}\left(\bx_{j}\right)^\T \left(\bx_{i}-\bx_{j}\right)\right)\right)(x_{i\ell}-x_{t\ell})\right]^2\\
&\le \frac{1}{n}\left[\frac{D^2}{n^2}\sum_{j=1}^n\sum_{i=1}^{n} \omega_{s}(\bx_i-\bx_j) (-L'\left(y_{i}\left(\bar{f}\left(\bx_{j}\right)+\bar{\bg}\left(\bx_{j}\right)^\T \left(\bx_{i}-\bx_{j}\right)\right)\right)^2\right]\\
&\le \frac{D^2}{n}\left[\frac{c_L}{n^2}\sum_{j=1}^n\sum_{i=1}^{n} \omega_{s}(\bx_i-\bx_j) L\left(y_{i}\left(\bar{f}\left(\bx_{j}\right)+\bar{\bg}\left(\bx_{j}\right)^\T \left(\bx_{i}-\bx_{j}\right)\right)\right)\right]\\
&= \frac{c_{L}D^2}{n} \hat{\calE}(\bar{f},\bar{\bg}) .
\end{align*}

Therefore we have $J_1(\bar{g}_\ell)\le \frac{ \sqrt{c_L} D}{\zeta \theta_\ell \sqrt{n}}\sqrt{\hat{\calE}(\bar{f},\bar{\bg})}$. Similarly we can show that $	J_0(\bar{f}_0) \le \frac{c_L}{2\zeta^2\theta_0 n}\hat{\calE}(\bar{f},\bar{\bg}) $. 	
\end{proof}
We now decompose the excess error as follows.


\begin{lemma} The following inequality holds for any $0<\epsilon\le 1$,
$$
\calE(\hat{f},\hat{\bg})-\calR_{s} +  	\lambda  J_0(f)+ \lambda  J_1(\bg)\le \mathscr{S}_1 +\left[1+\lambda c_L  \left( \frac{1}{2\theta_0\zeta^2n}  + \frac{Dp\epsilon}{\zeta\sqrt{n}}\right)\right]\mathscr{S}_2
+\mathscr{A}(\lambda ),
$$
where
$$
\mathscr{S}_2=\hat{\calE}(\hat{f},\hat{\bg})-\calE(\hat{f},\hat{\bg}) ,
$$
$$
\mathscr{S}_2=\hat{\calE}(\bar{f},\bar{\bg})-\calE(\bar{f},\bar{\bg}) ,
$$
and
\begin{align*}
\mathscr{A}(\lambda )= \left[1+\lambda c_L  \left( \frac{1}{2\theta_0\zeta^2n}  + \frac{Dp\epsilon}{\zeta\sqrt{n}}\right)\right]& \left(\calE(\bar{f},\bar{\bg})-\calR_s\right)\\
&+\lambda c_L  \left( \frac{1}{2\theta_0\zeta^2n}  + \frac{Dp\epsilon}{\zeta\sqrt{n}}\right)\calR_{s}+\frac{\lambda Dp}{\epsilon\zeta\sqrt{n}}.
\end{align*}

\end{lemma}
\begin{proof}
Note that
\begin{align*}	
\hat{\calE}(\hat{f},\hat{\bg}) +  	\lambda   J_0(f)+ \lambda  J_1(\bg) 
&=\hat{\calE}(\hat{f},\hat{\bg}) +	\lambda   J_0(\hat{f})+ \lambda  J_1(\hat{\bg})\\
&\le\hat{\calE}(\bar{f},\bar{\bg}) +	\lambda   J_0(\bar{f})+ \lambda  J_1(\bar{\bg})\\
&\le\hat{\calE}(\bar{f},\bar{\bg}) +	\lambda \frac{c_L}{2\zeta^2 \theta_0 n}\hat{\calE}(\bar{f},\bar{\bg})+
 \frac{\lambda  \sqrt{c_L} Dp}{\zeta \sqrt{n}}\sqrt{\hat{\calE}(\bar{f},\bar{\bg})}\\
&\le \left[1+\lambda c_L \left( \frac{1}{2\theta_0\zeta^2n}  + \frac{Dp\epsilon}{\zeta\sqrt{n}}\right)\right]\hat{\calE}(\bar{f},\bar{\bg})+ \frac{\lambda Dp}{\epsilon\zeta\sqrt{n}}\\
&\le \left[1+\lambda c_L  \left( \frac{1}{2\theta_0\zeta^2n}  + \frac{Dp\epsilon}{\zeta\sqrt{n}}\right)\right](\hat{\calE}(\bar{f},\bar{\bg})- \calE(\bar{f},\bar{\bg}))\\ 
&+\left[1+\lambda c_L  \left( \frac{1}{2\theta_0\zeta^2n}  + \frac{Dp\epsilon}{\zeta\sqrt{n}}\right)\right]\calE(\bar{f},\bar{\bg})+ \frac{\lambda Dp}{\epsilon\zeta\sqrt{n}}.
\end{align*}

The first inequality follows from the definition of $(\hat{f},\hat{\bg})$, the second inequality is from Lemma \ref{lem_pen}, the third inequality follows from the fact that $\sqrt{x} \le \epsilon x+1/\epsilon$ for any $\epsilon >0$. Adding $\calE(\hat{f},\hat{\bg}) -\hat{\calE}(\hat{f},\hat{\bg}) -\calR_{s}$ to both sides of inequality gives the desired results. 
\end{proof}

Without loss of generality, assume that the intermediate solution $(\bar{f}, \bar{\bg})$ in \eqref{int} is in the functional subspace 
$$\calF_{r_1}:=\left\{(f,\bg) \in \calH_{K}^{p+1}:
\theta_0\|f\|_{\calH_K}^2+\sum_{\ell=1}^p\sum_{\ell=1}^p\theta_{\ell}\|g_\ell\|_{\calH_K}^2\le r_1^2 \right\}$$ for some $r_1>0$. 
Similarly, define $$\calF_{r_2} :=\left\{f \in \calH_{\alpha}: \theta_0 \|f\|_{\calH_K}^2\le r_2^2\right\} \quad \text{ and } \quad \calF_{r_3}:=\left\{\bg\in \calH_{\alpha}^p := \sum_{\ell=1}^p \theta_{\ell}\|g_{\ell}\|_{\calH_K}\le r_3\right\}.$$ 
 We can control functions in these three functional spaces via a single radius by defining $R=\max\left\{r_1, r_2, r_3\right\}$.


In other words, for $(\bar{f},\bar{\bg}) \in \calF_{r_1}$, $\hat{f} \in \calF_{r_2}$, and $\hat{\bg} \in \calF_{r_3}$, followings hold simultaneously:
\begin{align*}
	\theta_0\|\bar{f}\|_{\calH_K}^2+\sum_{\ell=1}^p \theta_{\ell}  \|\bar{g}_\ell\|_{\calH_K}^2 &\le R^2\\\
	\theta_0\|\balpha_{0}\| &\le R^2\\
	\sum_{\ell=1}^p\sum_{\ell=1}^p\theta_\ell\|\balpha_{\ell}\|_2 &\le R.
\end{align*}

The quantity $\mathscr{S}_1$ and  $\mathscr{S}_2$ can be bounded by
\begin{align*}
S_1(R,\lambda )&= \sup_{f \in \calF_{r_{2}}, \bg \in \calF_{r_{3} }}|\calE({f},{\bg}) -\hat{\calE}({f},{\bg})|\\
S_2(R,\zeta)&= \sup_{(f,\bg) \in \calF_{r_{1}}}  |\calE({f},{\bg}) -\hat{\calE}({f},{\bg})|,
\end{align*}
respectively. That is, $\mathscr{S}_1 \le S(R,\lambda )$ and $\mathscr{S}_2 \le S(R,\zeta)$. By Lemma 25 of \citet{mukherjee2006estimation}, for every $R>0$, we have both

\begin{align}
P\left(|S_1(R,\lambda )-E(S_1(R,\lambda )) \ge \varepsilon   \right) \le 2 \exp\left(-\frac{n\varepsilon^2s^{2(p+2)}}{2M_R^2}\right) \label{eq_mc}\\
P\left(|S_2(R,\zeta)-E(S_2(R,\zeta)) \ge \varepsilon   \right) \le 2 \exp\left(-\frac{n\varepsilon^2s^{2(p+2)}}{2M_R^2}\right) \label{eq_mc2}.
\end{align}

\begin{lemma} \label{lem_esr}For every $R>0$,
$$
E\left(S(R,\lambda )\right)\le \frac{8L_R\left( R\left(\frac{1}{\theta_0}+\frac{2D}{c_4}\right)+L(0)\right)}{s^{p+2}\sqrt{n}}+\frac{2M_R}{ns^{p+2}},
$$
where
\begin{align*}
	L_R & = \max\left\{L'((1+D)R), L'(-(1+D)R) \right\},\\
	M_R & = \max\left\{L((1+D)R), L(-(1+D)R) \right\}.
\end{align*}
\end{lemma}
\begin{proof}
Denote ${\xi}(\bx,y,\bu)=w(\bx-\bu)L\left(y\left(f(\bu)+\bg(\bu)^\T(\bx-\bu)\right)\right).$ Then
\begin{align*}
	\calE(f,\bg)&=E_\bu E_{(\bx,y)}{\xi}(\bx,y,\bu)\\
	\hat{\calE}(f,\bg)&=\frac{1}{n^2}\sumi\sumj{\xi}(\bx_i,y,\bx_j).
\end{align*}
\begin{align*}
E(S(R,\lambda ))&\le \underset{(f,\bg)\in\calF_R}{\sup}\left|
	\calE(f,\bg)-\frac{1}{n}\sumj E_{(\bx,y)}{\xi}(\bx,y,\bx_j)
	\right|\\
	&\quad +\underset{(f,\bg)\in\calF_R}{\sup}\left|\frac{1}{n}\sumj E_{(\bx,y)}{\xi}(\bx,y,\bx_j)-\hat{\calE}(f,\bg)\right|\\
	&\le E_{(\bx,y)}\underset{(f,\bg)\in\calF_R}{\sup}\left|E_\bu{\xi}(\bx,y,\bu)-\frac{1}{n}\sumj{\xi}(\bx,y,\bx_j)\right|\\
	&\quad +\frac{1}{n}\sumj\underset{(f,\bg)\in\calF_R}{\sup}\underset{\bu\in\bX}{\sup}\left|
	E_{(\bx,y)}{\xi}(\bx,y,\bu)-\frac{1}{n-1}\sum_{i=1, i\neq j}^{n}{\xi}(\bx_i,y_i,\bu)
	\right|\\
	&\quad +\frac{1}{n}\sumj\left(\frac{1}{n}{\xi}(\bx_j,y_j,\bx_j)+\frac{1}{n(n-1)}\sum_{i=1, i\neq j}^{n}{\xi}(\bx_i,y_i,\bx_j)
	\right)\\&:=S_1+S_2+S_3
\end{align*}
Let $\epsilon_i$, $i=1,...,n$ be independent Rademacher variables. Denote
$$
G_{(\bx,y)}=\left\{h(\bu)=y\left[f(\bu)+\bg(\bu)^\T(\bx-\bu): (f,\bg)\in \calF_R
\right]\right\}, \qquad \forall (\bx,y) \in (\bX,Y)
$$
By Assumption A4, we see that $\theta_\ell > c_4$ for any $\ell$. Then it follows that
\begin{align*}
	ES_1&=E_{(\bx,y)}E\underset{h\in G_{(\bx,y)}}{\sup}\left|
	E_\bu\left[w(\bx-\bu)L(h(\bu))-\frac{1}{n}\sumj w(\bx-\bx_j)L(h(\bx_j))
	\right]\right|\\
	&\le 2\underset{(\bx,y) }{\sup}E\underset{h\in G_{(\bx,y)}}{\sup}\left|
	\frac{1}{n}\sumj \epsilon_j w(\bx-\bx_j)L(h(\bx_j))\right|\\
	&\le \frac{4}{s^{p+2}}\underset{(\bx,y) }{\sup}E\underset{h\in G_{(\bx,y)}}{\sup}
	\left|\frac{1}{n}\sumj\epsilon_j L((h(\bx_j)))
	\right|\\
	&\le \frac{4L_R}{s^{p+2}}\left(\underset{(\bx,y) }{\sup}E\underset{h\in G_{(\bx,y)}}{\sup} \left|\frac{1}{n}\sumj\epsilon_jh(\bx_j)\right|+\frac{L(0)}{\sqrt{n}}\right)\\
	&\le \frac{4L_R}{s^{p+2}}\left( \frac{R}{\theta_0}
	+\frac{ R}{c_4} \underset{\bx \in \bX}{\sup}E\left|\frac{1}{n}\sumj\epsilon_j\|\bx-\bx_j\|
	\right|+\frac{L(0)}{\sqrt{n}}\right)\\
	&\le \frac{4L_R\left(R \left(\frac{1}{\theta_0}+\frac{2D}{c_4}\right)+L(0)\right)}{s^{p+2}\sqrt{n}}.
\end{align*}
The fourth inequality follows from H\"{o}lder inequality and Assumption A1 and A4. 
Similarly, we can verify
$$
ES_2\le \frac{4L_R\left(R \left(\frac{1}{\theta_0}+\frac{2D}{c_4}\right)+L(0)\right)}{s^{p+2}\sqrt{n-1}}.
$$
Obviously, $S_3\le \frac{2M_R}{ns^{p+2}}$. Combining the estimates for $S_1, S_2$, and $S_3$ completes the proof.
\end{proof}
Similarly, we can verify 
$$
E\left(S_2(R,\lambda )\right)\le \frac{8L_R\left( R\left(\frac{1}{\sqrt{\theta_0}}+\frac{2D}{\sqrt{c_4}}\right)+L(0)\right)}{s^{p+2}\sqrt{n}}+\frac{2M_R}{ns^{p+2}}.
$$

Now we proof Theorem \ref{thm1}.
\begin{proof}[Proof of Theorem \ref{thm1}]
Together with \eqref{eq_mc}, \eqref{eq_mc2} and Lemma \ref{lem_esr} , we have with probability at least $1-\eta$, 
\begin{align} \label{s1s2}
\mathscr{S}_1 \le c_{s_1}\frac{L_RR+M_R\log\frac{2}{\eta}    }{\sqrt{n}s^{p+2}},\qquad
\mathscr{S}_2 \le c_{s_2}\frac{L_RR+M_R\log\frac{2}{\eta}    }{\sqrt{n}s^{p+2}}
\end{align}
for some $c_{s_1}, c_{s_2}>0$.

By Theorem 23 in \citet{mukherjee2006estimation}, if $(f^*, \nabla f^*) \in \calH_K^{p+1}$ and $(\bar{f},\bar{\bg}) \in \calF_{r_1}$, then with probability at least $1-\eta$, 
$$
\calE(\bar{f},\bar{\bg})-\calR_s \le \bar{c}\left(\frac{L_{r_1}r_1+M_{r_1}\log{\frac{2}{\eta}}}{\sqrt{n}{s^{p+2}}}+s^2+\zeta
\right),
$$
where
$$
r_1=\bar{c}\left\{1+\frac{s^{2}}{\zeta}+\left(\frac{L_{\zeta, s}}{\sqrt{\zeta s^{p+2}}}+M_{\zeta, s} \log \frac{2}{\eta}\right) \frac{1}{\sqrt{n \zeta s^{p+2}}}\right\}^{1 / 2}
$$
with some $\bar{c}\ge 1$, $L_{\zeta,s}=L_{\sqrt{2L(0)/\zeta s^{p+2}}}$, and $M_{\zeta,s}=M_{\sqrt{2L(0)/\zeta s^{p+2}}}$. By the definition of $\calR_{s}$ in \eqref{R1}, $\calR_{s} =\mathcal{O}(s^{-(p+2)})$. Take $\epsilon=s^{\frac{p+2}{2}}$, $\lambda =\zeta$, and $r_1 = R$. Then we have
\begin{align} \label{alambda}
\mathscr{A}(\lambda ) \le c_a&\left[\left(1+ \frac{c_L}{2\theta_0\lambda n}+\frac{Dps^{\frac{p+2}{2}}}{\sqrt{n}}\right) \left(\frac{L_{R}R+M_{R}\log{\frac{2}{\eta}}}{\sqrt{n}{s^{p+2}}}+{s^2}+{\lambda }
\right) \right. \nonumber \\
&\left. + \frac{c_L}{2\theta_0\lambda n} +   \frac{(c_L+1)Dp}{ s^{\frac{p+2}{2}}\sqrt{n} }\right].
\end{align}
By Theorem 15 in \citet{mukherjee2006estimation}, if $(f^*, \nabla f^*) \in \calH_K^{p+1}$ and $\hat{f} \in \calF_{r_2}$ and $\hat{\bg}) \in \calF_{r_3}$, we have

\begin{align*}
	\max \left\{\|\hat{f}-f^*\|_{L_{\mathrm{\rho}_{\mathcal{X}}}^{2}}^{2},\|\hat{\bg}-\nabla f^*\|_{L_{\mathrm{\rho}_{\mathcal{X}}}^{2}}^{2}\right\} \le \tilde{C}\left\{R^{2} s^{\tau}+Rs^{-\tau}\left(\calE(\hat{f},\hat{\bg})-R_s\right)\right\},
\end{align*}
for some $\tilde{C}>0$ and $R>1$. 
Therefore for some $\tilde{C}>0$ and $R>1$, we have
\begin{align} \label{thm1form}
	\max \left\{\|\hat{f}-f^*\|_{L_{\mathrm{\rho}_{\mathcal{X}}}^{2}}^{2},\|\hat{\bg}-\nabla f^*\|_{L_{\mathrm{\rho}_{\mathcal{X}}}^{2}}^{2}\right\} \le \tilde{C}\left\{R^{2} s^{\tau}+Rs^{-\tau}\left(\mathscr{E}(s,\lambda )\right)\right\}.
\end{align}
Now we verify $R>1$. Since $\calE(\hat{f},\hat{\bg})-\calR_{s}>0$,
\begin{align*}
\lambda  J_0(f)+ \lambda  J_1(\bg)&\le\calE(\hat{f},\hat{\bg})-\calR_{s} +  	\lambda  J_0(f)+ \lambda  J_1(\bg)\\
&\le  \mathscr{S}_1 +\left[1+\lambda c_L  \left( \frac{1}{2\theta_0\zeta^2n}  + \frac{Dp\epsilon}{\zeta\sqrt{n}}\right)\right]\mathscr{S}_2
+\mathscr{A}(\lambda )
\end{align*}
Furthermore, 
\begin{align} \label{lambdaeq}
	\lambda  J_0(f)+ \lambda  J_1(\bg)&\le\hat{\calE}(\hat{f},\hat{\bg})+  	\lambda  J_0(f)+ \lambda  J_1(\bg) \nonumber\\ 
       &\le\hat{\calE}(0,\mathbf{0})+  	\lambda  J_0(0)+ \lambda  J_1(\mathbf{0}) \nonumber\\
        &\le\frac{L(0)}{s^{p+2}}.
\end{align}
So $r_2^2+ r_3 \le\frac{L(0)}{\lambda s^{p+2}}$ and similarly we can show that $r_1^2 \le\frac{L(0)}{\lambda s^{p+2}}$. That is, $R^2\le\frac{L(0)}{\lambda s^{p+2}}$. 
Plug $R^2\le\frac{L(0)}{\lambda s^{p+2}}$ into \eqref{s1s2}, \eqref{alambda}, and \eqref{lambdaeq} yields that
\begin{align*}
R^2 &\le (c_{s_1}+c_{s_2}) \left(\frac{L_R}{\sqrt{\lambda s^{p+2}}} +M_R\log\frac{2}{\eta} \right)\frac{1}{\sqrt{n}\lambda s^{p+2}}\\
    &\quad +c_a\left[\left(1+ \frac{c_L+\sqrt{c_L}Dp}{ ns^{\frac{p+2}{2}}}\right)\left(\left(\frac{L_R}{\sqrt{\lambda s^{p+2}}} +M_R\log\frac{2}{\eta} \right)\frac{1}{\sqrt{n}\lambda s^{p+2}}+\frac{s^2}{\lambda }+1
    \right) \right.\\
    &\quad \left. + \frac{c_L}{2\theta_0\lambda n} +   \frac{(c_L+1)Dp}{ s^{\frac{p+2}{2}}\sqrt{n} }\right]\\
    &\le c_R\left[\left(1+ \frac{c_L+\sqrt{c_L}Dp}{ ns^{\frac{p+2}{2}}}\right)\left(\left(\frac{L_R}{\sqrt{\lambda s^{p+2}}} +M_R\log\frac{2}{\eta} \right)\frac{1}{\sqrt{n}\lambda s^{p+2}}+\frac{s^2}{\lambda }+1
    \right) \right. \\
    &\quad \left. + \frac{c_L}{2\theta_0\lambda n} +   \frac{(c_L+1)Dp}{ s^{\frac{p+2}{2}}\sqrt{n} }\right]
\end{align*}
for $c_R=\max\left\{c_{s_1}+c_{s_2},c_a\right\}$. Substituting the above $R$ into \eqref{thm1form} gives us the desired bound with the confidence at least $1-2\eta$. 

%
\end{proof}
\subsection{Proof of Theorem  \ref{thm2}}
\begin{proof}[Proof of theorem \ref{thm2}]
	First, we show that $\|\hat{\balpha}_{\ell}\|_2 =0$ for any $\ell >p_0$ by contradiction. Suppose that $\|\hat{\balpha}_{\ell}\|_2>0$ for some $\ell>p_0$. Taking the first derivatives of \eqref{eq:gradient} with respect to $\balpha_{\ell}$ yields that
	\begin{align}\label{eq1'}
		\frac{1}{n^2}	\sumi \sumj  \omega_{s}(\bx_{i}-\bx_{j})L' \{y_i (\hat{f}(\bx_j) + \hat{\bg}(\bx_j)^\T (\bx_i - \bx_j))\}(x_{i\ell}-x_{j\ell})\bk_{j}= -\frac{\lambda \theta_{\ell}\hat{\balpha_{\ell}}}{\|\hat{\balpha_{\ell}}\|_2}.
	\end{align}		
	Note that the norm of the right-hand side multiplied by 
	$s^{p+2}$ is 
	$s^{p+2}\lambda \theta_{\ell}$, which diverges to $\infty$ by Assumption A5. The contradiction can then be concluded by showing that the norm on the left-hand side of \eqref{eq1'} is less than 	$\mathcal{O}(n^\frac{1}{q}).$ Since each element of $\bk_j$ is bounded by Assumption A4, a slight modification of the proof of Lemma \ref{lem_pen} yields that 
	\begin{align*}
		\left\|\frac{1}{n^2}	\sumi \sumj  \omega_{s}(\bx_{i}-\bx_{j})L' \{y_i (\hat{f}(\bx_j) + \hat{\bg}(\bx_j)^\T (\bx_i - \bx_j))\}(x_{i\ell}-x_{j\ell})\bk_{j}\right\|_2\\
		\qquad \le \sqrt{c_{L}}D \sqrt{\hat{\calE}(\hat{f},\hat{\bg})} .
	\end{align*}
Similar argument of \eqref{lambdaeq} yields that $\calE_{'}(\hat{f},\hat{\bg})\le \frac{L(0)}{s^{p+2}}$.
Therefore the left-hand side of \eqref{eq1'}  is less bounded by 
$s^{\frac{p+2}{2}}\sqrt{c_L L(0)}D$, which converges to 0 as $s \to 0 $ and contradicts with the fact that $s^{p+2}  \lambda \theta_{\ell}$ diverges to $\infty$. Therefore, we have  $\|\hat{\balpha_{\ell}}\|_2=0$ for any $\ell>p_0$. 
Next, we show that $\|\hat{\balpha_{\ell}}\|_2>0$ for any $\ell\le p_0$. Define
$$
\mathcal{X}_s=\left\{\bx\in\bX:d(\bx,\partial\bX)>s \text{ and } \rho(\bx)\ge (1+c_\rho)s^\tau	\right\}.
$$
Same as the proof of Theorem 15 in \citet{mukherjee2006estimation}, for some given constant $c_g>0$, we have
$$
\int_{\mathcal{X}_s} \|\hat{\bg}(\bx)-\nabla f^*(\bx)\|d\rho_{\mathcal{X}}(\bx)\le c_gs^{-\tau}\left(r^2s^2+R\left[\calE(\hat{f},\hat{\bg})-R_s\right]\right).
$$
By Theorem \ref{thm1}, we have
$$
\int_{\mathcal{X}_s} \|\hat{\bg}(\bx)-\nabla f^*(\bx)\|d\rho_{\mathcal{X}}(\bx) \to 0
$$
as $n \to 0$. 
Now suppose $\|\hat{\balpha}_{\ell}\|_2=0$ for some $\ell<p_0$, which implies
$$
\int_{\mathcal{X}_s} \|\hat{\bg}(\bx)-\nabla f^*(\bx)\|d\Rho(\bx)=\int_{\mathcal{X}_s} \left\|\frac{\partial f^*(\bx)}{\partial x_\ell}\right\|_2^2 d\Rho(\bx).
$$
As $s \to 0$, 
$$
\int_{\mathcal{X}_s}\left\|\frac{\partial f^*(\bx)}{\partial x_\ell}\right\|_2^2 d\Rho(\bx)\ge 
\int_{\mathcal{X} \backslash \mathcal{X}_s} \left\|\frac{\partial f^*(\bx)}{\partial x_\ell}\right\|_2^2 d\Rho(\bx),
$$
which is a positive constant by Assumption A6 and then leads to the contradiction. Combining the above two statements implies the desired variable selection consistency. 
\end{proof}
